\documentclass[12pt]{article}

\usepackage{a4wide}
\usepackage{hans2012}

\newcommand{\lie}{\text{!`}}
\newcommand{\announce}{\text{!}}
\newcommand{\bluff}{\text{!`}!}

\usepackage[thmmarks]{ntheorem}
\theorembodyfont{\rmfamily}
\theoremsymbol{\ensuremath{\dashv}}
\usepackage{newproof}

\newtheorem{theorem}{Theorem}
\newtheorem{definition}[theorem]{Definition}
\newtheorem{proposition}[theorem]{Proposition}

\newcommand{\scep}{{\mathsf{sc}}}

\newcommand{\plau}{{\mathsf{pl}}}

\renewcommand{\Group}{A'}

\begin{document}

\title{The Ditmarsch Tale of Wonders --- The Dynamics of Lying}
\author{Hans van Ditmarsch, University of Seville, {\tt hvd@us.es}}
\date{\today}
\date{}

\maketitle

\begin{abstract}
We propose a dynamic logic of lying, wherein a `lie that $\phi$' (where $\phi$ is a formula in the logic) is an action in the sense of dynamic modal logic, that is interpreted as a state transformer relative to the formula $\phi$. The states that are being transformed are pointed Kripke models encoding the uncertainty of agents about their beliefs. Lies can be about factual propositions but also about modal formulas, such as the beliefs of other agents or the belief consequences of the lies of other agents. We distinguish (i) an outside observer who is lying to an agent that is modelled in the system, from (ii) one agent who is lying to another agent, and where both are modelled in the system. For either case, we further distinguish (iii) the agent who believes everything that it is told (even at the price of inconsistency), from (iv) the agent who only believes what it is told if that is consistent with its current beliefs, and from (v) the agent who believes everything that it is told by consistently revising its current beliefs. The logics have complete axiomatizations, which can most elegantly be shown by way of their embedding in what is known as action model logic or the extension of that logic to belief revision.
\end{abstract}

\section{Introduction} \label{sec.intro}

My favourite of Grimm's fairytales is `Hans im Gl\"uck' (Hans in Luck). A close second comes `The Ditmarsch Tale of Wonders'. In German this is called a `L\"ugenm\"archen', a `Liar's Tale'. It is as follows.

\begin{quote}
{\em 

  I will tell you something. I saw two roasted fowls flying; they flew
  quickly and had their breasts turned to Heaven and their backs to
  Hell; and an anvil and a mill-stone swam across the Rhine prettily,
  slowly, and gently; and a frog sat on the ice at Whitsuntide and ate
  a ploughshare.

  \medskip 

  Four fellows who wanted to catch a hare, went on crutches and
  stilts; one of them was deaf, the second blind, the third dumb, and
  the fourth could not stir a step. Do you want to know how it was
  done? First, the blind man saw the hare running across the field,
  the dumb one called to the deaf one, and the lame one seized it by
  the neck.

  \medskip 

  There were certain men who wished to sail on dry land, and they set
  their sails in the wind, and sailed away over great fields. Then
  they sailed over a high mountain, and there they were miserably
  drowned.

  \medskip 

  A crab was chasing a hare which was running away at full speed; and
  high up on the roof lay a cow which had climbed up there. In that
  country the flies are as big as the goats are here.

  \medskip 

  Open the window that the lies may fly out. \cite{grimm:1814}

%\bigskip 

%  {\em (Jacob Ludwig Grimm and Wilhelm Carl Grimm, Fairy Tales \cite{grimm:1814})}
}
\end{quote}

A passage like ``A crab was chasing a hare which was running away at full speed; and high up on the roof lay a cow which had climbed up there.''  contains very obvious lies. Nobody considers it possible that this is true. Crabs are reputedly slow, hares are reputedly fast.

\begin{quote} {\em In `The Ditmarsch Tale of Wonders', none of the lies are believed.} \end{quote}

In the movie `The Invention of Lying'\footnote{\url{http://en.wikipedia.org/wiki/The_Invention_of_Lying}} the main character Mark goes to a bank counter and finds out he only has \$300 in his account. But he needs \$800. Lying has not yet been invented in the 20th-century fairytale country of this movie --- that however seems to be in a universe very parellel to either the British Midlands or Brooklyn New York. Then and there, on the spot, Mark invents lying. We see some close-ups of Mark's braincells doing heavy duty---such a thing has not happened before. And then, Mark tells the bank employee assisting him that there must be a mistake: he has \$800 in his account. He is lying. She responds, oh well, then there must be a mistake with your account data, because on my screen it says you only have \$300. I'll inform system maintenance of the error. My apologies for the inconvenience. And she gives him \$800! In the remainder of the movie, Mark gets very rich.

Mark's lies are not as unbelievable as those in Grimm's fairytale. It is possible that he has \$800. It is just not true. Still, there is something unrealistic about the lies in this movie: new information is believed instantly. New information is even believed if it is inconsistent with prior information. After Mark's invention of lying while obtaining \$800, he's trying out his invention on many other people. It works all the time! There are shots wherein he first announces a fact, then its negation, then the fact again, while all the time his extremely credulous listeners keep believing every last announcement. New information is also believed if it contradicts direct observation. In a caf\'e, in company of several of his friends, he claims to be a one-armed bandit. And they commiserate with him, oh, I never knew you only had one arm, how terrible for you. All the time, Mark is sitting there drinking beer and gesturing with both hands while telling his story.

\begin{quote} {\em In the movie `The Invention of Lying', all lies are believed.} \end{quote}

In the real world, if you lie, sometimes other people believe you and sometimes they don't. When can you get away with a lie? Consider the consecutive numbers riddle \cite{littlewood:1953}. \begin{quote} {\em Anne and Bill are each going to be told a natural number. Their numbers will be one apart. The numbers are now being whispered in their respective ears. They are aware of this scenario. Suppose Anne is told 2 and Bill is told 3.

The following truthful conversation between Anne and Bill now takes place:
\begin{itemize}
\item Anne: ``I do not know your number.''
\item Bill: ``I do not know your number.''
\item Anne: ``I know your number.''
\item Bill: ``I know your number.''
\end{itemize}
Explain why is this possible.
} \end{quote} 

\noindent Initially, Anne is uncertain between Bill having 1 or 3, and Bill is uncertain between Anne having 2 or 4. So both Anne and Bill do not initially know their number. 

Suppose that Anne first says to Bill: ``I know your number.'' Anne is lying. Bill does not consider it possible that Anne knows his number. However, Anne did not know that Bill would not believe her. She considered it possible that Bill had 1, in which case Bill would have considered it possible that Anne was telling the truth, and would then have drawn the incorrect conclusion that Anne had 0.

Alternatively, suppose that the first announcement, Anne's, is truthful, but that Bill is lying in the second announcement and says to Anne: ``I know your number.'' If Anne believes that, she will then say ``I know your number,'' as she believes Bill to have 1. Her announcement is an honest mistake, because this belief is incorrect. However, as a result of Anne's announcement Bill will learn Anne's number, so that his announcement ``I know your number,'' that was a lie at the time, now has become true. 

That is, if you are still following us.

It seems not so clear how all this should be formalized in a logic interpreted on epistemic modal structures, and this is the topic of our paper.

\begin{quote} {\em In real life, some lies are believed and some are not.} \end{quote}

\subsection{The modal dynamics of lying}

What is a lie? Let $\atom$ be a proposition. You lie to me that $\atom$, if you believe that $\atom$ is false while you say that $\atom$, and with the intention that I believe $\atom$. The thing you say we call the announcement. If you succeed in your intention, I believe $\atom$, and I also believe that your announcement of $\atom$ was truthful, i.e., that you believed that $\atom$ when you said that $\atom$. In this investigation we abstract from the intentional aspect of lying. Such an abstraction seems reasonable. It is similar to that in AGM belief revision \cite{AGM:1985}, wherein one models how to incorporate new information in an agent's belief set, but abstracts from the process that made the new information acceptable to the agent. Our proposal is firmly grounded in modal logic. We employ dynamic epistemic logic \cite{hvdetal.del:2007}.

\paragraph*{Preconditions}
What are the modal preconditions and postconditions of a lie? Let us for now assume that $\atom$ itself is not a modal proposition but a Boolean proposition. We further assume two agents, $\agent$ and $\agentb$. Agent $\agent$ will be assumed female and agent $\agentb$ will be assumed male. Typically, in our exposition $\agent$ will be the speaker or sender and $\agentb$ will be the receiver or addressee. However, $\agent$ and $\agentb$ are not agent roles but agent names (and agent variables). We also model dialogue wherein agents speak in turn; so these roles may swap. Formula $B_\agent \atom$ stand for `agent $\agent$ believes that $\atom$'. We use belief modalities $B$ and not knowledge modalities $K$, because lying results in false beliefs, whereas knowledge modalities are used for correct beliefs. 

The precondition of `$\agent$ is lying that $\atom$ to $\agentb$' is $B_\agent \neg \atom$ ($\neg$ is negation). Stronger preconditions are conceivable, e.g., that the addressee considers it possible that the lie is true, $\neg B_\agentb \neg \atom$, or that the speaker believes that, $B_\agent \neg B_\agentb \neg \atom$. These conditions may not always hold while we still call the announcement a lie, because the speaker may not know whether the additional conditions are satisfied. We therefore will only require precondition $B_\agent \neg \atom$. 

We should contrast the announcement that $\atom$ by a lying agent with other forms of announcement. Just as a lying agent believes that $\atom$ is false when it announces $\atom$, a truthful agent believes that $\atom$ is true when it announces $\atom$. The precondition for a lying announcement by $\agent$ is $B_\agent \neg \atom$, and so the precondition for a truthful announcement by $\agent$ is $B_\agent \atom$. Here, were should put up some strong terminological barriers in order to avoid pitfalls and digressions into philosophy and epistemology. Truthful is synonymous with honest. Now dictionaries, that report actual usage, do not make a difference between an agent $\agent$ telling the truth and an agent $\agent$ believing that she is telling the truth. A modal logician has to make a choice. We mean the latter, exclusively. A truthful announcement may therefore not be a true announcement. If $\atom$ is false but agent $\agent$ mistakenly believes that $\atom$, then when she says $\atom$, that is a truthful but false announcement of $\atom$. Besides the truthful and the lying announcement there is yet another form of announcement, because in modal logic there are always three instead of two possibilities: either you believe $\atom$, or you believe $\neg \atom$, or you are uncertain whether $\atom$. The last corresponds to the precondition $\neg (B_\agent \atom \vel B_\agent \neg \atom)$ ($\vel$ is disjunction). An announcement wherein agent $\agent$ announces $\atom$ while she is uncertain about $\atom$ we propose to call a bluffing announcement. The dictionary meaning for the verb bluff is `to cause to believe what is untrue' or `to deceive or to feign' in a more general sense. Its meaning is even more intentional than that of lying. Feigning belief in $\atom$ means suggesting belief in $\atom$, by saying it, or otherwise behaving in accordance to it, although you do not have this belief. This corresponds to $\neg B_\agent \atom$ as precondition. This would make lying a form of bluffing, as $B_\agent \neg \atom$ implies $\neg B_\agent \atom$. It is common and according to Gricean conversational norms to consider that saying something that you believe to be false is worse than (or, at least, different from) saying something that you do not believe to be true. This brings us to $\neg B_\agent \atom \et \neg B_\agent \neg \atom$ ($\et$ is conjunction), equivalent to $\neg (B_\agent \atom \vel B_\agent \neg \atom)$. 

To the three mutually exclusive (and complete) preconditions $B_\agent \atom$, $B_\agent \neg\atom$, and $\neg(B_\agent \atom \vel B_\agent \neg\atom)$ we associate the truthful, lying, and bluffing announcement that $\atom$, and we call the announcing agent a truthteller, liar and bluffer, respectively.\footnote{Here again, it is stretching usage that a truthteller may not be telling the truth but only what she believes to be the truth, but that cannot be helped.} The three forms of announcement are intricately intertwined. This is obvious: to a credulous addressee a lying announcement appears to be a truthful announcement, whereas a skeptical addressee, who already believed the opposite of the announcement, has to make up his mind whether the speaker is merely mistaken, or is bluffing, or is even lying.

\paragraph*{Postconditions}
We now consider the postconditions of `$\agent$ is lying that $\atom$ to $\agentb$'. If $\agent$'s intention to deceive $\agentb$ is successful, $\agentb$ believes $\atom$ after the lie. Therefore, $B_\agentb \atom$ should be a postcondition of a successful execution of the action of lying. Also, the precondition should be preserved: $B_\agent \neg \atom$ should still true after the lie. In the first place, we propose logics to achieve this. However, this comes at a price. In case the agent $\agentb$ already believed the opposite, $B_\agentb \neg\atom$, then $\agentb$'s beliefs are inconsistent afterwards. (This merely means that $\agentb$'s accessibility relation is empty, not that the logic is inconsistent.) There are two different solutions for this: either $\agentb$ does not change his beliefs, so that $B_\agentb \neg\atom$ still holds after the lie, or the belief $B_\agentb \neg\atom$ is given up in order to consistently incorporate $B_\agentb \atom$. The three alternative postconditions after the lie that $\atom$ are therefore: (i) always make $B_\agentb \atom$ true after the lie (even at the price of inconsistency), (ii) only make $B_\agentb \atom$ true if agent $\agentb$ considered $\atom$ possible before the lie ($\neg B_\agentb \neg \atom$), and (iii) always make $B_\agentb \atom$ true by a consistency preserving process of belief revision. These are all modelled.

\paragraph*{Lying as a dynamic modal operator}
The preconditions and postconditions of lying may contain epistemic modal operators (belief modalities). The action of lying itself is modelled as a dynamic modal operator. The dynamic modal operator for `lying that $\atom$' is interpreted as an epistemic state transformer. An epistemic state is a pointed Kripke model (a model with a designated state) that encodes the beliefs of the agents. An epistemic action `agent $\agent$ lies that $\atom$ to agent $\agentb$' should transform an epistemic state satisfying $B_\agent \neg \atom$ into an epistemic state satisfying $B_\agentb \atom$ and $B_\agent \neg \atom$. The execution of such dynamic modal operators for epistemic actions depends on the initial epistemic state and that operator's description only. In that sense, they are different from dynamic modal operators for (PDL-style) actions that are interpreted using an accessibility relation in a given structure.

In this dynamic epistemic setting we can distinguish (i) the case of an external observer (an agent who is not explicitly modelled in the structures and in the logical language), who is lying to an agent modelled in the system, from (ii) the case of one agent lying to another agent, where both are explicitly modelled. For this external agent a truthful announcement is the same as a true announcement, and a lying announcement is the same as a false announcement.\footnote{This explains the terminological confusion in the area: the logic known as that of truthful public announcements is really the logic of {\em true} public announcements.} These matters will be addressed in detail.

In dynamic epistemic logics the transmission of messages is instantaneous and infallible. This is another assumption in our modelling framework.

\paragraph*{Lying about modal formulas}
The belief operators $B_\agent$ do not merely apply to Boolean propositions $\atom$ but to any proposition $\phi$ with belief modalities. This is known as higher-order belief. In the semantics, the generalization from `lying that $\atom$' to `lying that $\phi$' for any proposition, does not present any problem. This is, because it will be defined relative to the set of states where the formula is believed by the speaker $\agent$. For a subset of the domain, it does not matter if it was determined for a Boolean formula or for a modal formula. Still, there are other problems. 

Firstly, we aim for agents having what is known as `normal' beliefs, that satisfy consistency and introspection: $B_\agent \phi \imp \neg B_\agent \neg \phi$, $B_\agent \phi \imp B_\agent B_\agent \phi$, and $\neg B_\agent \phi \imp B_\agent \neg B_\agent \phi$. If we wish the addressee $\agentb$ to believe the lie that $\atom$ even if he already believed $\neg \atom$, his beliefs would become inconsistent. The property $B_\agent \phi \imp \neg B_\agent \neg \phi$ of consistent belief is therefore not preserved under lying updates. We will address this. 

Secondly, we may well require that $B_\agent \neg \atom$ is true before the lie by $\agent$ that $\atom$ and that $B_\agent \neg \atom$ and $B_\agentb \atom$ are true after the lie that $\atom$, but we cannot and do not even want to require for any $\phi$ that, if $B_\agent \neg \phi$ is true before the lie by $\agent$ that $\phi$, then $B_\agent \neg \phi$ and $B_\agentb \phi$ are true afterwards. For a typical example, suppose that $\atom$ is false and that $\agent$ lies that $\atom\et\neg B_\agentb \atom$ to $\agentb$.\footnote{It is not raining in Sevilla. You don't know that. I am lying to you: ``You don't know that it is raining in Sevilla!'' I.e., using the conventional conversational implicature, ``It is raining in Sevilla but/and you do not know that.'' This is a Moorean sentence.} We would say that the lie was successful if $B_\agentb \atom$ holds, not if $B_\agentb (\atom\et \neg B_\agentb \atom)$ holds, an inconsistency (for belief). Also, we do not want that $B_\agent \neg(\atom\et \neg B_\agentb \atom)$, that was true before the lie, remains true after the lie. It is common in the area to stick to the chosen semantic operation but not to require persistence of belief in such cases. 

Our proposed modelling, that also applies to modal formulas, allows us to elegantly explain why in the consecutive number riddle we can have that (see above) `{\em as a result of Anne's announcement Bill will learn Anne's number, so that his prior announcement ``I know your number,'' that was a lie at the time, now has become true.}' The seemingly contradictory announcements in the riddle are Moorean phenomenona, and the added aspect of uncertainty about lying or truthtelling makes the analysis more complex, and the results more interesting.

\paragraph*{More modalities} In the concluding Section \ref{sec.further} we discuss further issues in the modal logic of lying, such as group epistemic operators (common belief) in preconditions and postconditions of lying, and structures with histories of actions to keep track of the number of past lies.

\subsection{A short history of lying}

We conclude this introduction with an review of literature on lying. 

\paragraph*{Philosophy}
Lying has been a thriving topic in the philosophical community for a long, long time \cite{siegler:1966,bok:1978,mahon:2006,mahon.stanford:2008}. Almost any analysis starts with quoting Augustine on lying: \begin{quote} ``that man lies, who has one thing in his mind and utters another in words'' \\ ``the fault of him who lies, is the desire of deceiving in the uttering of his mind'' \cite{Augustine:dm} \end{quote} In other words: saying that $\atom$ while believing that $\neg\atom$, with the intention to make believe $\atom$, our starting assumption. The requirements for the belief preconditions and postconditions in such works are illuminating \cite{mahon.stanford:2008}. For example, the addressee should not merely believe the lie but believe it to be believed by the speaker. Indeed, ... and even believed to be commonly believed, would the modal logician say (see the final Section \ref{sec.further}). Scenarios involving eavesdroppers (can you lie to an agent who is not the addressee?) are relevant for logic and multi-agent system design, and also claims that you can only lie if you really say something: an omission is not a lie \cite{mahon.stanford:2008}. Wrong, says the computer scientist: if the protocol is common knowledge, you can lie by not acting when you should have; say, by not stepping forward in the muddy children problem, although you know that you are muddy. The philosophical literature also clearly distinguishes between false propositions and propositions believed to be false but in fact true, so that when you lie about them, in fact you tell the truth. Gettier-like scenarios are presented, including delayed justification \cite{rott:2003a}.\footnote{Suppose that you believe that $\neg \atom$ and that you lie that $\atom$. Later find out that your belief was mistaken because $\atom$ was really true. You can then with some justification say ``Ah, so I was not really lying.''} Much is said on the morality of lying \cite{bok:1978} and on its intentional aspect. As said, we abstract from the intentional aspect of lying. We also abstract from its moral aspect.

\paragraph*{Psychology}
Lying excites great interest in the general public. Lots of popular science books are written on the topic, typical examples are \cite{trivers:2011,velden:2011}. In psychology, biology, and other experimental sciences lying and deception are related. A cuckoo is `lying' if it is laying eggs in another bird's nest. Two issues are relevant for our investigation. Firstly, that it is typical to be believed, and that lying is therefore the exception. We model the `successful' lie that is indeed believed, unless there is evidence to the contrary: prior belief in the opposite. Secondly, that the detection of lying is costly, and that this is a reason to be typically believed. In logic, cost is computational complexity. The issue of the complexity of lying is shortly addressed in Section \ref{sec.further} on further research.

\paragraph*{Economics}
In economics, `cheap talk' is making false promises. Your talk is cheap if you do not intend to execute an action that you publicly announced to plan. It is therefore a lie, it is deception \cite{gneezy:2005,kartiketal:2007}. Our focus is different. We do not model lying about planned actions but lying about propositions, and in particular on their belief consequences. Economists postulate probabilities for lying strategies and truthful strategies, to be tested experimentally. We only distinguish lies that are always believed from lies that (in the face of contradictory prior belief) are never believed.

\paragraph*{Logic}
Papers that model lying as an epistemic action, inducing a transformation of an epistemic model, include \cite{baltag:2002,steiner:2006,baltagetal.tlg:2008,hvd.comments:2008,kooietal:2011,hvdetal.lying:2011}. Lying by an external observer has been discussed by Baltag and collaborators from the inception of dynamic epistemic logic onward \cite{baltag:2002}; the later \cite{baltagetal.tlg:2008} also discusses lying in logics with plausible belief, as does \cite{hvd.comments:2008}. In \cite{hvdetal.lying:2011} the conscious update in \cite{gerbrandyetal:1997} is applied to model lying by an external observer. In \cite{sakamaetal:2010} the authors give a modal logic of lying and bluffing, including intentions. Instead of bluffing they call this bullshit, after \cite{frankfurt:2005}. Strangely, in view of contradictory Moorean phenomena, they do not model lying as a dynamic modality. In \cite{steiner:2006,kooietal:2011} the unbelievable update is considered; this is the issue consistency preservation for belief, as in our treatment of unbelievable lies (rejecting the lie that $\atom$ if you already believe $\neg\atom$). The promising manuscript \cite{liuetal:2012} allows explicit reference in the logical language to truthful, lying and bluffing agents (the authors call this `agent types'), thus enabling some form of self-reference.

\paragraph*{Artificial intelligence?}
Various of the already cited works could have been put under this header. But then it would not have been a question mark. Applications of epistemic logic in artificial intelligence typically are about knowledge and not about belief. Successful frameworks as interpreted systems and knowledge programs \cite{faginetal:1995} model multi-${\mathcal S5}$ systems. Our analysis of the modal dynamics of lying aims to prepare the ground for AI applications involving belief instead of knowledge, and to facilitate determining the complexities of such reasoning tasks.\footnote{Argumentation theory, when seen as an area in AI, does of course model beliefs and their justificiations.}

\subsection{Contributions and overview}

A main and novel contribution of our paper is a precise model of the informative consequences of two agents lying to each other, and a logic for that, including a treatment of bluffing. This agent-to-agent-lying, in the logic called agent announcement logic, is presented in Section \ref{sec.agent}, with an extended example in Section \ref{sec.example}. A special, simpler, case is that of an outside observer who is lying to an agent that is modelled in the system. This (truthful and lying) public announcement logic is treated in Section \ref{sec.pub}. That section mainly contains results from \cite{hvdetal.lying:2011}. Section \ref{sec.prelim} introduces the standard truthful (without lying) public announcement logic of which our proposals can be seen as variations. Section \ref{sec.am} on action models is an alternative perspective on the frameworks presented in Section \ref{sec.pub} and Section \ref{sec.agent}. It anchors them in another part of dynamic epistemic logic. Section \ref{sec.unbel} contains another novel contribution. It adapts the logics of the Sections \ref{sec.pub} and \ref{sec.agent} to the requirement that unbelievable lies (if you hear $\atom$ but already believe $\neg\atom$) should not be incorporated. Subsequently, Section \ref{sec.plaus} adapts these logics to the requirement that unbelievable lies, on the contrary, should be incorporated, but consistently so. This can be anchored in yet another part of the dynamic epistemic logical literature, involving structures with plausibility relations. All these logics have complete axiomatizations (which is unremarkable). An incidental novel contribution is how to resolve ambiguity between bluffing and lying with disjunctive normal forms (Proposition \ref{prop.adf}). 
%, and how to distinguish a lie from a mistake on the assumption of background knowledge (Proposition \ref{prop.mistake}). 
Section \ref{sec.further} finishes the paper with considerations on the limitations of our approach and further research.

\section{Truthful public announcements} \label{sec.prelim}

The well-known logic of truthful public announcements \cite{plaza:1989,baltagetal:1998} is an extension of multi-agent epistemic logic. Its language, structures, and semantics are as follows. Given are a finite set of agents $\Agents$ and a countable set of propositional variables $\Atoms$. 
\begin{definition}
The {\em language of truthful public announcement logic} is inductively defined as \[  \lang(!) \ \ni \  \phi ::= \atom \ | \ \neg \phi \ | \ (\phi \et \psi) \ | \ B_\agent \phi \ | \ [!\phi]\psi \] where $\atom \in \Atoms$, and $\agent \in \Agents$.\footnote{Unlike in the introduction, $\atom$ is a Boolean/propositional variable, and not any Boolean proposition.} Without the announcement operators we get the language $\lang$ of epistemic logic. \end{definition} Other propositional connectives are defined by abbreviation. For $B_\agent \phi$, read `agent $\agent$ believes formula $\phi$'. Agent variables are $\agent,\agentb,\agentc,\dots$. For $[!\phi] \psi$, read `after truthful public announcement of $\phi$, formula $\psi$ (is true)'. The dual operator for the necessity-type announcement operator is by abbreviation defined as $\dia{!\phi} \psi := \neg [!\phi] \neg \psi$. If $B_\agent \neg\phi$ we say that $\phi$ is {\em unbelievable} and, consequently, if $\neg B_\agent \neg\phi$ we say that $\phi$ is {\em believable}. This is also read as `agent $\agent$ considers it possible that $\phi$'. 

\begin{definition}
An {\em epistemic model} $M = ( \States, R, V )$ consists of a {\em domain} $\States$ of {\em states} (or `worlds'), an {\em accessibility function} $R: \Agents \imp {\mathcal P}(\States \times \States)$, where each $R(\agent)$, for which we write $R_\agent$, is an accessibility relation, and a {\em valuation} $V: \Atoms \imp {\mathcal P}(\States)$, where each $V(\atom)$ represents the set of states where $\atom$ is true. For $\state \in \States$, $(M,\state)$ is an {\em epistemic state}. \end{definition} An epistemic state is also known as a pointed Kripke model. We often omit the parentheses. Four model classes will appear in this work. Without any restrictions we call the model class ${\mathcal K}$. The class of models where all accessibility relations are transitive and euclidean is called ${\mathcal K45}$, and if they are also serial it is called ${\mathcal KD45}$. The class of models where all accessibility relations are equivalence relations is ${\mathcal S5}$. Class ${\mathcal KD45}$ is said to have the {\em properties of belief}, and ${\mathcal S5}$ to have the {\em properties of knowledge}.
\begin{definition}
Assume an epistemic model $M = ( \States, R, V )$.  
\[ \begin{array}{lcl}
M,\state \models \atom &\mbox{iff} & \state \in V_\atom \\ 
M,\state \models \neg \phi &\mbox{iff} & M,\state \not \models \phi \\ 
M,\state \models \phi \et \psi &\mbox{iff} & M,\state \models \phi  \text{ and } M,\state \models \psi \\  
M,\state \models B_\agent \phi &\mbox{iff} & \mbox{for all \ }  \stateb \in \States: R_\agent(\state,\stateb) \text{ implies } M,\stateb  \models \phi \\  
M,\state \models [!\phi] \psi &\mbox{iff} & M,\state \models \phi  \text{ implies } M|\phi,\state \models \psi  \end{array} \] where the model restriction $M|\phi = ( \States', R', V'  )$ is defined as $\States' = \{ \state' \in \States \suchthat M,\state' \models \phi \}$ (= $\II{\phi}_M$), $R'_\agent = R_\agent \inter \ (\States' \times \States')$ and $V'(\atom) = V(\atom) \inter \States'$. \end{definition} A complete proof system for this logic for class ${\mathcal S5}$ is presented in \cite{plaza:1989}. Trivial variations are complete axiomatizations for the model classes ${\mathcal K}$ and ${\mathcal K45}$. The interaction axiom between announcement and belief is: 
\begin{definition} \label{def.annbelief} \[ \begin{array}{lcl}
\mbox{} [\announce \phi] B_\agent \psi & \leftrightarrow &  \phi \rightarrow B_\agent [\announce \phi] \psi 
\end{array} \] \end{definition}
The interaction between announcement and other operators we assume known \cite{hvdetal.del:2007}. It changes predictably in the other logics we present. In the coming sections, we will only vary the dynamic part of the logic, and focus on that completely.

For an example of the semantics of public announcement, consider a situation wherein an agent $\agentb$ is uncertain about $\atom$, and receives the information that $\atom$. The initial uncertainty requires a model consisting of two states, one where $\atom$ is true and one where $\atom$ is false. In view of the continuation, we draw all accessibility relations. For convenience, a state has been given the value of the atom true there as its name. The actual state is underlined.

\bigskip
%\bigskip

%\begin{equation}
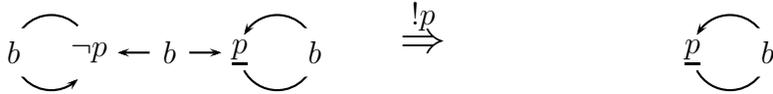
\psset{border=2pt, nodesep=4pt, radius=2pt, tnpos=a}
\pspicture(-1,0)(2,0)
$
\rput(0,0){\rnode{00}{\neg\atom}}
\rput(2,0){\rnode{10}{\underline{\atom}}}
\ncline{<->}{00}{10} \ncput*{\agentb}
\nccircle[angle=90]{->}{00}{.5} \ncput*{\agentb}
\nccircle[angle=270]{->}{10}{.5} \ncput*{\agentb}
$
\endpspicture
\hspace{2cm} {\Large $\stackrel {! \atom} \Imp$} \  
\psset{border=2pt, nodesep=4pt, radius=2pt, tnpos=a}
\pspicture(-1,0)(2,0)
$
\rput(2,0){\rnode{10}{\underline{\atom}}}
\nccircle[angle=270]{->}{10}{.5} \ncput*{\agentb}
$
\endpspicture
%\end{equation}
\bigskip
\bigskip

\noindent 
In the actual state $\atom$ is true, and from the actual state two states are accessible: that $\atom$-state and the $\neg\atom$-state. Therefore, the agent does not believe $\atom$ (as there is an accessible $\neg\atom$-state) and does not believe $\neg\atom$ either (as there is an accessible $\atom$-state). She is uncertain whether $\atom$. The announcement $!\atom$ results in a restriction of the epistemic state to the $\atom$-state (it is false in the other state). On the right, the agent believes that $\atom$. In the initial epistemic state it is therefore true that $\atom \et \neg (B_\agentb \atom \vel B_\agentb \neg \atom) \et [!\atom] B_\agentb \atom$. Note that both on the left and on the right the accessibility relation is an equivalence relation. Indeed, the class ${\mathcal S5}$ is closed under truthful public announcements. The example models change of knowledge rather than the weaker change of belief.

The class ${\mathcal KD45}$ is not closed under truthful public announcements. Let us consider another example. Suppose agent $\agent$ incorrectly believes $\atom$ and wishes to process the truthful public announcement that $\neg\atom$.

\bigskip

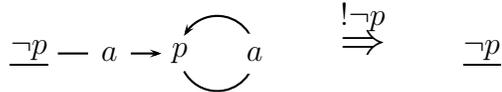
\psset{border=2pt, nodesep=4pt, radius=2pt, tnpos=a}
\pspicture(-1,0)(2,0)
$
\rput(0,0){\rnode{00}{\underline{\neg\atom}}}
\rput(2,0){\rnode{10}{\atom}}
\ncline{->}{00}{10} \ncput*{\agent}
\nccircle[angle=270]{->}{10}{.5} \ncput*{\agent}
$
\endpspicture
\hspace{2cm} {\Large $\stackrel {! \neg \atom} \Imp$} \  
\psset{border=2pt, nodesep=4pt, radius=2pt, tnpos=a}
\pspicture(-1,0)(2,0)
$
\rput(0,0){\rnode{00}{\underline{\neg\atom}}}
$
\endpspicture

\bigskip
\bigskip

\noindent On the left, we have that $\neg \atom \et B_\agent \atom$ is true. The new information $!\neg \atom$ results in eliminating the $\atom$-state, and consequently agent $\agent$'s accessibility relation becomes empty. She believes everything. The model on the left is ${\mathcal KD45}$, but the model on the right is not ${\mathcal KD45}$, it is not serial. However, the class ${\mathcal K45}$ is closed under truthful public announcements.

\section{Logic of truthful and lying public announcements} \label{sec.pub}

The logic of lying public announcements complements the logic of truthful public announcements. They are inseparably tied to one another. For a clear link we need to use an alternative semantics for truthful public announcement logic. The results in this section are mainly from \cite{hvdetal.lying:2011}.

We expand the language of truthful public announcement logic with another inductive construct $[\lie\phi]\psi$, for `after lying public announcement of $\phi$, formula $\psi$ (is true)'; in short `after the lie that $\phi$, $\psi$'. This is the language $\lang(!,\lie)$.
\begin{definition}[Language] \[ \lang(!,\lie) \ \ni \ \phi ::= \atom \ | \ \neg \phi \ | \ (\phi \et \psi) \ | \ B_\agent \phi \ | \ [!\phi]\psi \ | \ [\lie\phi]\psi \] \end{definition}
Truthful public announcement logic is the logic to model the revelations of a benevolent god, taken as the truth without questioning. The announcing agent is not modelled in public announcement logic, but only the effect of its announcements on the audience, the set of all agents. Consider a {\em false} public announcement, made by a malevolent entity, the devil. Everything he says is false. Everything is a lie. As in religion, god and the devil are inseparable and should be modelled simultaneously.

As a semantics for this logic we employ an alternative to the semantics for public announcement logic from the previous section. This alternative is the semantics of conscious updates \cite{gerbrandyetal:1997}.\footnote{A public announcement of $\phi$ is a conscious update with the test $?\phi$ for the entire set of agents $\Agents$ (such updates can be for any subset of agents). In \cite{gerbrandyetal:1997} these updates are interpreted on non-wellfounded sets, namely on rooted infinite trees, such as the tree unwinding of a pointed $S5$ model. We present the simpler semantics for conscious update of \cite{kooi.jancl:2007}. We note that \cite{gerbrandyetal:1997} and \cite{plaza:1989} were independent proposals for the semantics of public announcements.} When announcing $\phi$, instead of eliminating states where $\phi$ does not hold, one eliminates for each agent those pairs of its accessibility relation where $\phi$ does not hold in the second argument of the pair. The effect of the announcement of $\phi$ is that only states where $\phi$ is true are accessible for the agents. In other words, the semantics is arrow eliminating instead of state eliminating. We call this the {\em believed public announcement}. New information is accepted by the agents independent from the truth of that information.
\begin{definition}[Semantics of believed public announcement] \label{def.belpubann}
\[ \begin{array}{lcl}
M,\state \models [\text{\it believed public announcement that }\phi] \psi &\mbox{iff} & M^\phi,\state \models \psi \end{array} \] where epistemic model $M^\phi$ is as $M$ except that (with $\States$ the domain of $M$) \[ R^\phi_\agent \ := \ R_\agent \inter \ (\States \times \II{\phi}_M) . \] 
\end{definition}
In \cite{hvdetal.lying:2011}, the believed public announcement of $\phi$ is called manipulative update with $\phi$. The original proposal there is to view believed public announcement of $\phi$ as non-deterministic choice (as in action logics and PDL-style logics) between truthful public announcement of $\phi$ and lying public announcement of $\phi$.
\begin{definition}[Semantics of truthful and lying public announcement] \label{def.truthlyingpub}
\[ \begin{array}{lcl}
M,\state \models [!\phi] \psi &\mbox{iff} & M,\state \models \phi  \text{ implies } M^\phi,\state \models \psi  \\
M,\state \models [\lie\phi] \psi &\mbox{iff} & M,\state \models \neg\phi  \text{ implies } M^\phi,\state \models \psi 
\end{array} \] \end{definition}
If we now define $[!\phi\lie]\psi$ by abbreviation as $[!\phi]\psi \et [\lie\phi]\psi$, then $!\phi\lie$ has the semantics of `believed public announcement that $\phi$'. The non-determinism of this operator (it has two executions, one when $\phi$ is true and one when $\phi$ is false) comes to the fore if we write it in the dual form: $\dia{!\phi\lie}\psi \eq \dia{!\phi}\psi \vel \dia{\lie\phi}\psi$; in other words, we can view $!\phi\lie$ as some $!\phi\union\lie\phi$, where $\union$ is non-deterministic choice.

The following result justifies that it is not ambiguous to write $!\phi$ for `arrow elimination' truthful public announcement but also for `state elimination' truthful public announcement. \begin{proposition}[\cite{kooi.jancl:2007,hvdetal.lying:2011}] \label{prop.statearrow} Let $\state\in\Domain(M)$ and $M,\state\models\phi$. Then \[ (M^\phi,\state) \bisim (M|\phi, \state) . \] \end{proposition} The symbol $\bisim$ stands for `is bisimilar to', a well-known notion that guarantees that the models cannot be distinguished in the logical language \cite{blackburnetal:2001}.

State elimination seems simpler than arrow elimination. Having gone through the trouble of reinterpreting truthful public announcement in arrow elimination semantics, why did we not proceed in the other direction and reinterpret lying public announcement in state elimination semantics? This is not possible! Consider a model wherein all belief is correct, i.e., a model with equivalence relations encoding the beliefs of the agents. A state elimination semantics preserves equivalence, whereas lying inevitably introduces false beliefs: states not accessible to themselves.

The axiom for belief after truthful public announcement remains what it was (Definition \ref{def.annbelief}, and using Proposition \ref{prop.statearrow}) and the axiom for the reduction of belief after lying is as follows.
\begin{definition}[Axiom for belief after lying \cite{hvdetal.lying:2011}] \label{def.axiombellie}
\[ \begin{array}{lcl} 
%\mbox{} [\lie \phi] p & \leftrightarrow &  \neg \phi \rightarrow p  \\
%\mbox{} [\lie \phi] \neg \psi & \leftrightarrow &  \neg \phi \rightarrow \neg [\lie \phi] \psi \\
%\mbox{} [\lie \phi] (\psi_1 \land \psi_2) & \leftrightarrow &  [\lie \phi] \psi_1 \land  [\lie \phi] \psi_2 \\
\mbox{} [\lie \phi] B_\agent \psi &  \leftrightarrow & \neg \phi \rightarrow B_\agent [\announce \phi] \psi .
\end{array} \] \end{definition}
After the lying public announcement that $\phi$, agent $\agent$ believes that $\psi$, if and only if, on condition that $\phi$ is false, agent $\agent$ believes that $\psi$ after truthful public announcement that $\phi$. To the credulous person who believes the lie, the lie appears to be the truth.

\begin{proposition}[\cite{hvdetal.lying:2011}] \label{prop.publiclying}
The axiomatization of the logic of truthful and lying public announcements is complete (for the model class ${\mathcal K}$ and for the model class ${\mathcal K}45$).
\end{proposition}
\begin{proof} 
As for the logic of truthful public announcements, completeness is shown by a reduction argument. All formulas in $\lang(!,\lie)$ are equivalent to formulas in $\lang$ (epistemic logic). By means of equivalences such as in the axiom for belief after lying one can rewrite each formula to an equivalent one without announcement operators. (As the class ${\mathcal K}45$ is not closed under updates with announcements, the logic is not complete for that class.)
\end{proof}
For an example, we show the effect of truthful and lying announcement of $\atom$ to an agent $\agentb$ in the model with uncertainty about $\atom$. The actual state must be different in these models: when lying that $\atom$, $\atom$ is in fact false, whereas when truthfully announcing that $\atom$, $\atom$ is in fact true. For lying we get

\bigskip

\psset{border=2pt, nodesep=4pt, radius=2pt, tnpos=a}
\pspicture(-1,-0)(2,0)
$
\rput(0,0){\rnode{00}{\underline{\neg\atom}}}
\rput(2,0){\rnode{10}{\atom}}
\ncline{<->}{00}{10} \ncput*{\agentb}
\nccircle[angle=90]{->}{00}{.5} \ncput*{\agentb}
\nccircle[angle=270]{->}{10}{.5} \ncput*{\agentb}
$
\endpspicture
\hspace{1.5cm} {\Large $\stackrel {\lie \atom} \Imp$} \ 
\psset{border=2pt, nodesep=4pt, radius=2pt, tnpos=a}
\pspicture(-1,0)(2,0)
$
\rput(0,0){\rnode{00}{\underline{\neg\atom}}}
\rput(2,0){\rnode{10}{\atom}}
\ncline{->}{00}{10} \ncput*{\agentb}
\nccircle[angle=270]{->}{10}{.5} \ncput*{\agentb}
$
\endpspicture

\bigskip
\bigskip

\noindent whereas for truthtelling we get

\bigskip

\psset{border=2pt, nodesep=4pt, radius=2pt, tnpos=a}
\pspicture(-1,0)(2,0)
$
\rput(0,0){\rnode{00}{\neg\atom}}
\rput(2,0){\rnode{10}{\underline{\atom}}}
\ncline{<->}{00}{10} \ncput*{\agentb}
\nccircle[angle=90]{->}{00}{.5} \ncput*{\agentb}
\nccircle[angle=270]{->}{10}{.5} \ncput*{\agentb}
$
\endpspicture
\hspace{1.5cm} {\Large $\stackrel {! \atom} \Imp$} \ 
\psset{border=2pt, nodesep=4pt, radius=2pt, tnpos=a}
\pspicture(-1,0)(2,0)
$
\rput(0,0){\rnode{00}{\neg\atom}}
\rput(2,0){\rnode{10}{\underline{\atom}}}
\ncline{->}{00}{10} \ncput*{\agentb}
\nccircle[angle=270]{->}{10}{.5} \ncput*{\agentb}
$
\endpspicture

\bigskip
\bigskip

The reduction principle in \cite{gerbrandyetal:1997,kooi.jancl:2007} for the interaction between belief and believed announcement is, in terms of our language, $[! \phi \lie] B_\agent \phi \eq B_\agent (\phi \imp [! \phi \lie]\psi)$. This seems to have a different shape, as the modal operator binds the entire implication. But it is indeed valid in our semantics (a technical detail we did not find elsewhere).

\begin{proposition} \label{prop.believed}
$\models [! \phi \lie] B_\agent \phi \eq B_\agent (\phi \imp [! \phi \lie]\psi)$
\end{proposition}
\begin{proof}

\[ \begin{array}{lll}
[! \phi \lie] B_\agent \psi & \Eq & [! \phi]B_\agent \psi \et [\lie \phi] B_\agent \psi \\ & \Eq^\sharp & 
(\phi \imp B_\agent [! \phi]\psi) \et (\neg\phi \imp B_\agent[! \phi] \psi) \\ & \Eq & 
B_\agent [! \phi]\psi \\ & \Eq & 
B_\agent (\phi \imp [! \phi]\psi) \\ & \Eq^* & 
B_\agent ((\phi \imp [! \phi]\psi) \et (\phi \imp [\lie \phi] \psi)) \\ & \Eq & 
B_\agent (\phi \imp ([! \phi]\psi \et [\lie \phi] \psi)) \\ & \Eq & 
B_\agent (\phi \imp [! \phi \lie] \psi)
\end{array} \]
The $\sharp$-ed equivalence holds because from the semantics of truthful and lying announcement directly follows that $[!\phi]\psi \eq (\phi \imp [!\phi]\psi)$ and $[\lie\phi]\psi \eq (\neg\phi \imp [\lie\phi]\psi)$. The *-ed equivalence holds because
\[ \begin{array}{lll}
\phi \imp [\lie \phi] \psi & \Eq & \phi \imp (\neg \phi \imp [\lie \phi] \psi) \\ & \Eq & (\phi \et\neg \phi) \imp [\lie \phi] \psi \\ & \Eq & \F \imp [\lie \phi] \psi \\ & \Eq & \T
\end{array} \]
\end{proof}
The logic of truthful and lying public announcement satisfies the property `substitution of equivalents', which we have used repeatedly in the proof of Proposition \ref{prop.believed}, but the logic does not satisfy the property `substitution of variables' (substitution of formulas for propositional variables preserves validity). For example, $[!\atom]\atom$ is valid but, clearly, $[!(\atom\et\neg B_\agent\atom)](\atom\et\neg B_\agent\atom)$ is invalid.

\section{Agent announcement logic} \label{sec.agent}

In the logic of lying and truthful public announcements, the announcing agent is an outside observer and is implicit. Therefore, it is also implicit that it believes that the announcement is false or true. In multi-agent epistemic logic, it is common to formalize `agent $\agent$ truthfully announces $\phi$' as `the outside observer truthfully announces $B_\agent \phi$'. However, `agent $\agent$ lies that $\phi$' cannot be modelled as `the outside observer lies that $B_\agent \phi$'. This is the main reason for a logic of lying.

For a counterexample, consider an epistemic state with equivalence relations, modelling knowledge, where $\agentb$ does not know whether $\atom$, $\agent$ knows whether $\atom$, and $\atom$ is true. Agent $\agent$ is in the position to tell $\agentb$ the truth about $\atom$. A truthful public announcement of $B_\agent \atom$ (arrow elimination semantics, Definition \ref{def.truthlyingpub}) indeed simulates that $\agent$ truthfully and publicly announces $\atom$. 

\bigskip

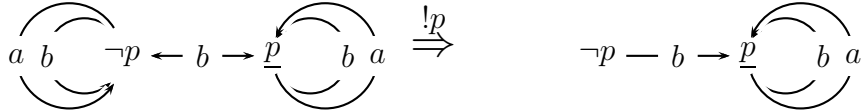
\psset{border=2pt, nodesep=4pt, radius=2pt, tnpos=a}
\pspicture(-1,0)(3.5,0)
$
\rput(0,0){\rnode{00}{\neg\atom}}
\rput(2,0){\rnode{10}{\underline{\atom}}}
\ncline{<->}{00}{10} \ncput*{\agentb}
\nccircle[angle=90]{->}{00}{.5} \ncput*{\agentb}
\nccircle[angle=270]{->}{10}{.5} \ncput*{\agentb}
\nccircle[angle=90]{->}{00}{.7} \ncput*{\agent}
\nccircle[angle=270]{->}{10}{.7} \ncput*{\agent}
$
\endpspicture
\hspace{0.2cm} {\Large $\stackrel {! \atom} \Imp$} \hspace{0.6cm}
\psset{border=2pt, nodesep=4pt, radius=2pt, tnpos=a}
\pspicture(-1,0)(3.5,0)
$
\rput(0,0){\rnode{00}{\neg\atom}}
\rput(2,0){\rnode{10}{\underline{\atom}}}
\ncline{->}{00}{10} \ncput*{\agentb}
%\nccircle[angle=90]{->}{00}{.5} \ncput*{\agentb}
\nccircle[angle=270]{->}{10}{.5} \ncput*{\agentb}
%\nccircle[angle=90]{->}{00}{.7} \ncput*{\agent}
\nccircle[angle=270]{->}{10}{.7} \ncput*{\agent}
$
\endpspicture

\bigskip
\bigskip

Given the same model, now suppose $\atom$ is false, and that $\agent$ lies that $\atom$. A lying public announcement of $B_\agent \atom$ (it satisfies the required precondition $\neg B_\agent \atom$) does not result in the desired information state, because this makes agent $\agent$ believe her own lie. In fact, as she already knew $\neg\atom$, this makes $\agent$'s beliefs inconsistent.

\bigskip

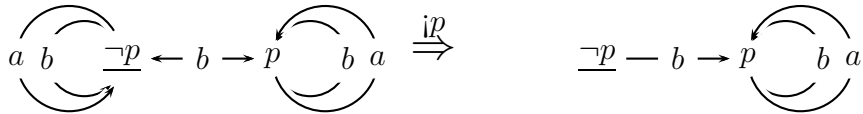
\psset{border=2pt, nodesep=4pt, radius=2pt, tnpos=a}
\pspicture(-1,0)(3.5,0)
$
\rput(0,0){\rnode{00}{\underline{\neg\atom}}}
\rput(2,0){\rnode{10}{\atom}}
\ncline{<->}{00}{10} \ncput*{\agentb}
\nccircle[angle=90]{->}{00}{.5} \ncput*{\agentb}
\nccircle[angle=270]{->}{10}{.5} \ncput*{\agentb}
\nccircle[angle=90]{->}{00}{.7} \ncput*{\agent}
\nccircle[angle=270]{->}{10}{.7} \ncput*{\agent}
$
\endpspicture
\hspace{0.2cm} {\Large $\stackrel {\lie \atom} \Imp$} \hspace{0.6cm}
\psset{border=2pt, nodesep=4pt, radius=2pt, tnpos=a}
\pspicture(-1,0)(3.5,0)
$
\rput(0,0){\rnode{00}{\underline{\neg\atom}}}
\rput(2,0){\rnode{10}{\atom}}
\ncline{->}{00}{10} \ncput*{\agentb}
%\nccircle[angle=90]{->}{00}{.5} \ncput*{\agentb}
\nccircle[angle=270]{->}{10}{.5} \ncput*{\agentb}
%\nccircle[angle=90]{->}{00}{.7} \ncput*{\agent}
\nccircle[angle=270]{->}{10}{.7} \ncput*{\agent}
$
\endpspicture

\bigskip
\bigskip

\noindent Instead, a lie by $\agent$ to $\agentb$ that $\atom$ should have the following effect:

\bigskip

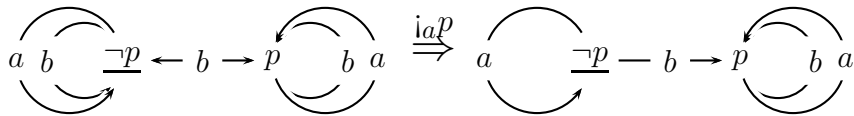
\psset{border=2pt, nodesep=4pt, radius=2pt, tnpos=a}
\pspicture(-1,0)(3.5,0)
$
\rput(0,0){\rnode{00}{\underline{\neg\atom}}}
\rput(2,0){\rnode{10}{\atom}}
\ncline{<->}{00}{10} \ncput*{\agentb}
\nccircle[angle=90]{->}{00}{.5} \ncput*{\agentb}
\nccircle[angle=270]{->}{10}{.5} \ncput*{\agentb}
\nccircle[angle=90]{->}{00}{.7} \ncput*{\agent}
\nccircle[angle=270]{->}{10}{.7} \ncput*{\agent}
$
\endpspicture
\hspace{.2cm} {\Large $\stackrel {\lie_\agent \atom} \Imp$} \hspace{0.5cm}
\psset{border=2pt, nodesep=4pt, radius=2pt, tnpos=a}
\pspicture(-1,0)(3.5,0)
$
\rput(0,0){\rnode{00}{\underline{\neg\atom}}}
\rput(2,0){\rnode{10}{\atom}}
\ncline{->}{00}{10} \ncput*{\agentb}
%\nccircle[angle=90]{->}{00}{.5} \ncput*{\agentb}
\nccircle[angle=270]{->}{10}{.5} \ncput*{\agentb}
\nccircle[angle=90]{->}{00}{.7} \ncput*{\agent}
\nccircle[angle=270]{->}{10}{.7} \ncput*{\agent}
$
\endpspicture

\bigskip
\bigskip

\noindent After this lie we have that $\agent$ still believes that $\neg\atom$, but that $\agentb$ believes that $\atom$. (We even have that $\agentb$ believes that $\agentb$ and $\agent$ have common belief of $\atom$, see Section \ref{sec.further}.) We satisfied the requirements of a lying agent announcement, for believed lies. 

The precondition for agent $\agent$ truthtelling that $\phi$ is $B_\agent \phi$ and the precondition for agent $\agent$ lying that $\atom$ is $B_\agent \neg\phi$. Another form of announcement is {\em bluffing}. You are bluffing that $\phi$, if you say that $\phi$ but are uncertain whether $\phi$. The precondition for agent $\agent$ bluffing is therefore $\neg (B_\agent \phi \vel B_\agent \neg \phi)$. 
%The three exclusive preconditions for announcing $\phi$ are: $B_\agentb \phi$, $B_\agentb \neg\phi$, and $\neg (B_\agentb \phi \vel B_\agentb \neg \phi)$, respectively the preconditions for truthtelling, lying, and bluffing. 
If belief is implicit, we had only two preconditions for announcing $\phi$: $\phi$ and $\neg \phi$, for truthtelling and lying. The `third' would be $\neg (\phi \vel \neg \phi)$ which is $\F$. The devil can lie, but it cannot bluff.

The postconditions of these three types of announcement are a function of their effect on the accessibility relation of the agents. The effect is the same for all three types. However, it is different for the speaker and for the addressee(s). \begin{itemize} \item States where $\phi$ was believed by speaker $\agent$ remain accessible to speaker $\agent$; \item
States where $\phi$ was not believed by speaker $\agent$ remain accessible to speaker $\agent$; \item States where $\phi$ was believed by speaker $\agent$ remain accessible to addressee $\agentb$; \item States where $\phi$ was not believed by speaker $\agent$ are no longer accessible to addressee $\agentb$. \end{itemize} 
These requirements are embodied by the following syntax and semantics.

%\bigskip
\begin{definition} \label{def.langagent}
The {\em language of agent announcement logic} is defined as
\[ \lang(!_\agent,\lie_\agent,\bluff_\agent) \ \ni \ \phi ::= \atom \ | \ \neg \phi \ | \ (\phi \et \psi) \ | \ B_\agent \phi \ | \ [!_\agent \phi] \psi \ | \ [\lie_\agent \phi] \psi \ | \ [\bluff_\agent \phi] \psi \] 
\end{definition}
The inductive constructs $[!_\agent \phi] \psi$,$[\lie_\agent \phi] \psi$, and $[\bluff_\agent \phi] \psi$ stand for, respectively, $\agent$ truthfully announces $\phi$, $\agent$ is lying that $\phi$, and $\agent$ is bluffing that $\phi$; where agent $\agent$ addresses all other agents $\agentb$. 
\begin{definition} \label{def.semagenta}
\[ \begin{array}{lcl}
M,\state \models [!_\agent\phi] \psi &\mbox{iff} & M,\state \models B_\agent \phi \text{ implies } M^\phi_\agent,\state \models \psi  \\
M,\state \models [\lie_\agent\phi] \psi &\mbox{iff} & M,\state \models B_\agent \neg\phi  \text{ implies } M^\phi_\agent,\state \models \psi \\
M,\state \models [\bluff_\agent\phi] \psi &\mbox{iff} & M,\state \models \neg (B_\agent \phi \vel B_\agent \neg\phi) \text{ implies } M^\phi_\agent,\state \models \psi 
\end{array} \]  where $M^\phi_\agent$ is as $M$ except for the accessibility relation $R'$ defined as ($\States$ is the domain of $M$, and $\agent \neq \agentb$)
 \[ \begin{array}{rcl}
R'_\agent & := & R_\agent \\
R'_\agentb & := & R_\agentb \inter \ (\States \times \II{B_\agent \phi}_M).
\end{array} \]
\end{definition}
The principles for $\agent$ truthtelling, lying, or bluffing to $\agentb$ are as follows. The belief consequences for the speaker are different from the belief consequences for the addressee(s).
\begin{definition}[Axioms for the belief consequences of announcements]
\begin{eqnarray*} 
\mbox{} [!_\agent \phi] B_\agentb \psi & \leftrightarrow & B_\agent \phi \rightarrow B_\agentb [\announce_\agent \phi] \psi \\
\mbox{} [!_\agent \phi] B_\agent \psi & \leftrightarrow & B_\agent \phi \rightarrow B_\agent [!_\agent \phi] \psi \\
\mbox{} [\lie_\agent \phi] B_\agentb \psi & \leftrightarrow & B_\agent \neg \phi \rightarrow B_\agentb [\announce_\agent \phi] \psi \\
\mbox{} [\lie_\agent \phi] B_\agent \psi & \leftrightarrow & B_\agent \neg \phi \rightarrow B_\agent [\lie_\agent \phi] \psi \\
\mbox{} [\bluff_\agent \phi] B_\agentb \psi & \leftrightarrow & \neg (B_\agent \phi \vel B_\agent \neg\phi) \rightarrow B_\agentb [\announce_\agent \phi] \psi \\
\mbox{} [\bluff_\agent \phi] B_\agent \psi & \leftrightarrow & \neg (B_\agent \phi \vel B_\agent \neg\phi) \rightarrow B_\agent [\bluff_\agent \phi] \psi
\end{eqnarray*}
\end{definition}
In other words, the liar knows that he is lying, but the dupe he is lying to, believes that the liar is telling the truth. The principles for truthtelling and bluffing are similar to that for lying, but with the obvious different conditions on the right hand side of the equivalences. The bluffer knows that he is bluffing, but the dupe he is bluffing to, believes that the bluffer is telling the truth. And in case the announcing agent is truthful, there is no discrepancy, both the speaker and the addressee believe that the consequences are those of truthtelling.
\begin{proposition} \label{prop.axiomagent}
The axiomatization of the logic of agent announcements is complete.
\end{proposition}
\begin{proof}
Just as in the previous logics (see Proposition \ref{prop.publiclying}) completeness is shown by a reduction argument. All formulas in $\lang(!_\agent,\lie_\agent,\bluff_\agent)$ are equivalent to formulas in epistemic logic. In the axioms above, the announcement operator is (on the right) always pushed further down into any given formula. As before, the result holds for model classes ${\mathcal K}$, ${\mathcal K}45$ and ${\mathcal S}5$, but not for ${\mathcal KD}45$.

An alternative, indirect, completeness proof is that agent announcement logic is action model logic for a given action model, as will be explained in Section \ref{sec.am}. 
\end{proof}

\noindent As an example illustrating the difference between a truthtelling, lying and bluffing agent announcement we present the following model wherein the adressee $\agentb$, hearing the announcement of $\atom$ by agent $\agent$, considers all three possible. In fact, $\agent$ is bluffing, and $\agent$'s announcement that $\atom$ is false. After the announcement, $\agentb$ incorrectly believes that $\atom$, but $\agent$ is still uncertain whether $\atom$. If the bottom-left state had been the actual state, $\agent$ would have been lying that $\atom$, and if the bottom-right state had been the actual state, it would have been a truthful announcement by $\agent$ that $\atom$.

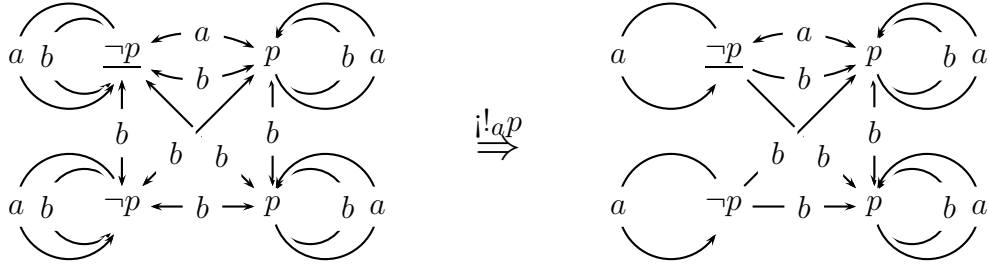
\psset{border=2pt, nodesep=4pt, radius=2pt, tnpos=a}
\pspicture(-1,-1)(3.5,3)
$
\rput(0,0){\rnode{00}{\neg\atom}}
\rput(2,0){\rnode{10}{\atom}}
\rput(0,2){\rnode{01}{\underline{\neg\atom}}}
\rput(2,2){\rnode{11}{\atom}}
\ncline{<->}{00}{10} \ncput*{\agentb}
\ncarc[arcangle=25]{<->}{01}{11} \ncput*{\agent}
\ncarc[arcangle=-25]{<->}{01}{11} \ncput*{\agentb}
\ncline{<->}{00}{01} \ncput*{\agentb}
\ncline{<->}{10}{11} \ncput*{\agentb}
\ncline{<->}{01}{10} \ncput*[npos=0.7]{\agentb}
\ncline{<->}{00}{11} \ncput*[npos=0.3]{\agentb}
\nccircle[angle=90]{->}{01}{.5} \ncput*{\agentb}
\nccircle[angle=270]{->}{11}{.5} \ncput*{\agentb}
\nccircle[angle=90]{->}{01}{.7} \ncput*{\agent}
\nccircle[angle=270]{->}{11}{.7} \ncput*{\agent}
\nccircle[angle=90]{->}{00}{.5} \ncput*{\agentb}
\nccircle[angle=270]{->}{10}{.5} \ncput*{\agentb}
\nccircle[angle=90]{->}{00}{.7} \ncput*{\agent}
\nccircle[angle=270]{->}{10}{.7} \ncput*{\agent}
\rput(5,1){\rnode{imp}{\text{\Large $\stackrel {\bluff_\agent \atom} \Imp$}}}
\rput(8,0){\rnode{00x}{\neg\atom}}
\rput(10,0){\rnode{10x}{\atom}}
\rput(8,2){\rnode{01x}{\underline{\neg\atom}}}
\rput(10,2){\rnode{11x}{\atom}}
\ncline{->}{00x}{10x} \ncput*{\agentb}
\ncline{->}{01x}{10x} \ncput*[npos=0.7]{\agentb}
\ncline{->}{00x}{11x} \ncput*[npos=0.3]{\agentb}
\ncarc[arcangle=25]{<->}{01x}{11x} \ncput*{\agent}
\ncarc[arcangle=-25]{->}{01x}{11x} \ncput*{\agentb}
%\ncline{<->}{00x}{01x} \ncput*{\agentb}
\ncline{<->}{10x}{11x} \ncput*{\agentb}
%\nccircle[angle=90]{->}{01x}{.5} \ncput*{\agentb}
\nccircle[angle=270]{->}{11x}{.5} \ncput*{\agentb}
\nccircle[angle=90]{->}{01x}{.7} \ncput*{\agent}
\nccircle[angle=270]{->}{11x}{.7} \ncput*{\agent}
%\nccircle[angle=270]{->}{00x}{.5} \ncput*{\agentb}
\nccircle[angle=270]{->}{10x}{.5} \ncput*{\agentb}
\nccircle[angle=90]{->}{00x}{.7} \ncput*{\agent}
\nccircle[angle=270]{->}{10x}{.7} \ncput*{\agent}
$
\endpspicture

\paragraph*{Unbelievable announcements}
The axiomatization of the logic of agent announcements is incomplete for ${\mathcal KD}45$ (see the proof of Proposition \ref{prop.axiomagent}) because of the problem of unbelievable announcements. In Sections \ref{sec.unbel} and \ref{sec.plaus} we present alternative logics wherein believable announcements (announcements of $\phi$ to addressee $\agentb$ such that $\neg B_\agentb \neg \phi$ is true) are treated differently from unbelievable announcements (such that $B_\agentb \neg \phi$ is true). These logics are complete for class ${\mathcal KD}45$.

\paragraph*{Outside observer}
Consider a depiction of an epistemic model. The outside observer is the guy or girl looking at the picture: you, the reader. She can see all different states. She has no uncertainty and her beliefs are correct. It is therefore that her truthful announcements are true and that her lying announcements are false. It is also therefore that `truthful public announcement logic' is not a misnomer, it is indeed the logic of how to process new information that is true. 

We can model the outside observer as an agent $gd$, for `{\em g}od or the {\em d}evil'. (The model does not need to be a bisimulation contraction.)
\begin{proposition} \label{prop.gd}
Given an epistemic model $M$, let $gd \in \Agents$ be an agent with an accessibility relation that is the identity on $M$. Then $M^\phi_{gd} = M^\phi$. Let $\phi,\psi$ not contain announcement operators, then $[!_{gd} \phi]\psi$ is equivalent to $[! \phi]\psi$, and $[\lie_{gd} \phi]\psi$ is equivalent to $[\lie \phi]\psi$.
\end{proposition}
\begin{proof}
In Definition \ref{def.semagenta}, the accessibility of the addressees is adjusted to $R'_\agentb := R_\agentb \inter \ (\States \times \II{B_{gd} \phi}_M)$. As $R_{gd}$ is the identity, $B_{gd} \phi$ is equivalent to $\phi$. So, $R'_\agentb := R_\agentb \inter \ (\States \times \II{\phi}_M)$, as in Definition \ref{def.belpubann}.
\end{proof}

\paragraph*{Lying about beliefs}

Agents may announce factual propositions (Boolean formulas) but also modal propositions, and thus be lying and bluffing about them. In the consecutive numbers riddle the announcements `I know your number' and `I do not know your number' are modal propositions, and the agents may be lying about those.\footnote{In social interaction, untruthfully announcing {\em epistemic} propositions is not always considered lying with the negative moral connotation. Suppose we work in the same department and one of our colleagues, $X$, is having a divorce. I know this. I also know that you know this. But we have not discussed the matter between us. I can bring up the matter in conversation by saying `You know that $X$ is having a divorce!'. But this is unwise. You may not be willing to admit your knowledge, because $X$'s husband is your friend, which I have no reason to know; etc. A better strategy for me is to say `You may not know that $X$ is having a divorce'. This is a lie. I do not consider it possible that you do not know that. But, unless we are very good friends, you will not laugh in my face to that and respond with `Liar!'. Could it be that lies about facts are typically considered worse than lies about epistemic propositions, and that the more modalities you stack in your lying announcement, the more innocent the lie becomes?} (Details in Section \ref{sec.example}.)

For our target agents, that satisfy introspection (so that $B_\agent B_\agent \phi \eq B_\agent \phi$ and $B_\agent \neg B_\agent \phi \eq \neg B_\agent \phi$ are validities), the distinction between bluffing and lying seems to become blurred. If I am uncertain whether $\atom$, I would be bluffing if I told you that $\atom$, but I would be lying if I told you that I believe that $\atom$. The announcement that $\atom$ satisfies the precondition $\neg (B_\agent \atom \vel B_\agent \neg \atom)$. It is bluffing that $\atom$ (it is $\bluff_\agent \atom$). But the announcement that $B_\agent \atom$ satisfies the precondition $B_\agent \neg B_\agent \atom$, the negation of the announcement. It is lying that $B_\agent \atom$ (it is $\lie_\agent B_\agent\atom$). (Proof: $\neg (B_\agent \atom \vel B_\agent \neg \atom)$ is equivalent to $\neg B_\agent \atom \et \neg B_\agent \neg \atom$, from which follows $\neg B_\agent \atom$, and with negative introspection $B_\agent \neg B_\agent \atom$.) We would prefer to call both bluffing, and `$\agent$ announces that $B_\agent \atom$' strictly or really `$\agent$ announces that $\atom$'. A general solution to avoid such ambiguity involves more than merely stripping a formula of an outer $B_\agent$ operator: $\agent$ announcing that $B_\agent B_\agent \atom$ should also strictly be $\agent$ announcing that $\atom$, and $\agent$ announcing that $B_\agent\atom\et B_\agent\atomb$ should strictly be $\agent$ announcing that that $\atom\et\atomb$. We need recursion.
In the following definition and proposition we assume an agent with consistent beliefs.

\begin{definition}[Strictly lying]
An announcement by $\agent$ that $\phi$ is {\em strict} iff $\phi$ is equivalent to $\T$ or $\phi$ is equivalent to $\Vel_i\psi_i^\agent \et \Et_i \neg B_\agent \neg \psi_i^\agent$ where all $\psi_i^\agent$ are alternating disjunctive forms (see proof below) for agents other than $\agent$. An agent announcement $\lie_\agent \phi$ is {\em strictly lying} iff there is a $\phi'$ equivalent to $\phi$ such that $\phi'$ is strict and $\lie_\agent \phi'$ is a lying agent announcement. (Similarly for {\em strictly bluffing}.)
\end{definition}
\begin{proposition} \label{prop.adf}
For each $\phi \in\lang(!_\agent,\lie_\agent,\bluff_\agent)$ there is an equivalent $\psi\in\lang$ that is strict.
\end{proposition}
\begin{proof} 
We first define the {\em alternating disjunctive form} ($adf$). An $adf$ is a disjunction of formulas of shape $\psi_0 \et \Et_{\agentb\in\Group} (B_\agentb \Vel_i \psi_i^\agentb \et \Et_i \neg B_\agentb \neg \psi_i^\agentb)$, where $\psi_0 \in \lang$, $\Group\subseteq\Agents$, and where $\psi_i^\agentb$ is an $adf$ for a group $\Group'\subset\Agents$ of agents other than $\agentb$ (the `alternating' in alternating disjunctive forms). The indices $i$ range over $1,\dots,n$ ($n \geq 1$). We may assume that for each $\agentb$ some $\psi_i^\agentb$ is non-trivial (i.e., not equivalent to $\T$), because otherwise we can take $\Group\setminus\agentb$ in the conjunction above. Each formula in multi-agent ${\mathcal KD45}$ is equivalent to an $adf$ \cite{hales:2011}.\footnote{Reported in \cite{hales:2011}, an $adf$ is a multi-agent generalization of the (single-agent) disjunctive form as in \cite[p.35]{meyeretal:1995} (where it is called ${\mathcal S5}$ normal form), and a special case of the disjunctive form as in \cite{dagostinoetal:2008}. An $adf$ contains no stacks of $B_\agent$ operators without an intermediary $B_\agentb$ operator for another agent, i.e., if $B_\agent \chi$ is a subformula of an $adf$ and $B_\agent \chi''$ is a subformula of $\chi$ then there is an agent $\agentb\neq\agent$ and a $\chi'$ such that $B_\agentb \chi'$ is a subformula of $\chi$ and $B_\agent \chi''$ is a subformula of $\chi'$.}

Now for the proof. Let $\phi \in\lang(!_\agent,\lie_\agent,\bluff_\agent)$. Determine an equivalent $\phi'\in\lang$ (i.e., rewrite $\phi$ into a multi-agent epistemic formula, without announcement operators). Let $\psi$ be an $adf$ equivalent to $B_\agent \phi'$. It is elementary to see that $\psi$ must be either equivalent to $\T$ or have the form $B_\agent \Vel_i \psi_i^\agent \et \Et_i \neg B_\agent \neg \psi_i^\agent$ and where some $\psi_i^\agent$ is non-trivial. We now observe that
\[ \begin{array}{lll} \psi & \Eq & B_\agent \Vel_i \psi_i^\agent \et \Et_i \neg B_\agent \neg \psi_i^\agent \\
& \Eq & B_\agent \Vel_i \psi_i^\agent \et \Et_i B_\agent \neg B_\agent \neg \psi_i^\agent)  \\
& \Eq & B_\agent (\Vel_i \psi_i^\agent \et \Et_i \neg B_\agent \neg \psi_i^\agent)
\end{array} \]
and that $\Vel_i \psi_i^\agent \et \Et_i \neg B_\agent \neg \psi_i^\agent$ is strict.
\end{proof}

\paragraph*{Mistakes and lies}
There are two sorts of mistaken beliefs. 

The first kind is when I believe $\atom\imp\neg\atomb$ and I believe $\atom$, and, when questioned about my belief in $\atomb$, confidently state that I belief that $\atomb$. This is a mistaken belief. The formula $B_\agent\atomb$ should be false, because $B_\agent\neg\atomb$ is a deductive consequence of $B_\agent (\atom\imp\neg\atomb)$ and $B_\agent\atom$. To explain this sort of mistaken belief we need to resort to bounded rationality. We do not do that. 

The second kind of mistaken belief is when I believe that $\atom$ but when in  fact $\atom$ is false. Fully rational agents can make such mistakes. This, we can address. It amounts to the truth of $\neg\atom \et B_\agent \atom$.

What is the difference between a lie and a mistake? I am lying if I say $\phi$ and believe $\neg \phi$ (independently from the truth of $\phi$), whereas I am mistaken if I say $\phi$ and believe $\phi$, but $\phi$ is false. Let $\agent$ be the speaker, then the precondition of lying that $\phi$ is $B_a \neg \phi$ and the precondition of a mistaken truthful announcement that $\phi$ is $\neg \phi \et B_a \phi$. The speaker can therefore distinguish a lie from a mistake.

How about the addressee? Clearly, a ${\mathcal KD45}$ agent cannot distinguish another agent lying from being mistaken, because it can itself be mistaken about the announced proposition and about the speaker's beliefs of that proposition. We can then only observe that if $\agent$ says $\phi$ given that $B_\agentb (\neg\phi\et B_\agent\phi)$, then $\agentb$ believes $\agent$ to be mistaken, whereas if $\agent$ says $\phi$ given that $B_\agentb B_\agent\neg\phi)$, then $\agentb$ believes $\agent$ to be lying. 

But if both knowledge and belief play a role then the addressee can distinguish a lie from a mistake. A standard assumption in many games (therefore called `fair games') and in many other distributed systems, is initial common knowledge of all agents of the uncertainty about the system. (I {\em know} --- I can assume to know --- that you do not know my card. If not, it's not fair, it was cheating.) Part of that initial common knowledge is of facts  or otherwise positive formulas, persisting after every update. In such a system, beliefs that are not knowledge can only be induced by incorrect updates (lies). For such an addresssee, a mistake is that $\agent$ says $\phi$ when $K_\agentb (\neg\phi\et B_\agent\phi)$, whereas a lie is that $\agent$ says $\phi$ when $K_\agentb K_\agent\neg\phi$. The consecutive numbers riddle gives an example of such a lie and such a mistake.

\section{Lying in the consecutive numbers riddle} \label{sec.example}

In this section we give an extended example of agent announcement logic. 
We recall the consecutive numbers riddle. \begin{quote} {\em Anne and Bill are each going to be told a natural number. Their numbers will be one apart. The numbers are now being whispered in their respective ears. They are aware of this scenario. Suppose Anne is told 2 and Bill is told 3.

The following truthful conversation between Anne and Bill now takes place:
\begin{itemize}
\item Anne: ``I do not know your number.''
\item Bill: ``I do not know your number.''
\item Anne: ``I know your number.''
\item Bill: ``I know your number.''
\end{itemize}
Explain why is this possible.
} \end{quote} 

\noindent We first analyze the riddle in public announcement logic with the state elimination semantics. Although the riddle goes back to 1953 \cite{littlewood:1953}, this analysis dates from 1984 \cite{Emde:84} and played an important role in the development of the area of dynamic epistemic logic.\footnote{In \cite{Emde:84} the consecutive numbers riddle is called the Conway Paradox, after \cite{conwayetal:1977}. That name that has since stuck. Conway (and Paterson) indeed present an epistemic riddle in \cite{conwayetal:1977}, but not the consecutive numbers riddle. This was a mistaken attribution, as Peter van Emde Boas (an author of \cite{Emde:84}) told us.} The initial model encodes that the numbers are consecutive and consists of two disconnected countable parts, one where the agents have common knowledge (correct common belief) that $\agent$'s number is odd and $\agentb$'s number is even, and another one where the agents have common knowledge that $\agent$'s number is even and $\agentb$'s number is odd. As before, the actual state is underlined. This is the state $(2,3)$, where $\agent$ is told 2 and $\agentb$ is told 3. In this simplified visualization we use the convention that symmetry and reflexivity are assumed.

\bigskip

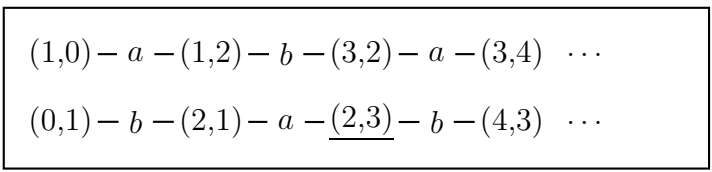
\psframebox{
\psset{border=1pt, nodesep=2pt, radius=2pt, tnpos=a}
\pspicture(-0.5,-0.2)(8.5,1.7)
\rput(0,0.3){\rnode{00}{(0,1)}}
\rput(2,0.3){\rnode{10}{(2,1)}}
\rput(4,0.3){\rnode{20}{\underline{(2,3)}}}
\rput(6,0.3){\rnode{30}{(4,3)}}
\rput(7,0.3){\rnode{40}{\dots}}
\rput(0,1.2){\rnode{01}{(1,0)}}
\rput(2,1.2){\rnode{11}{(1,2)}}
\rput(4,1.2){\rnode{21}{(3,2)}}
\rput(6,1.2){\rnode{31}{(3,4)}}
\rput(7,1.2){\rnode{41}{\dots}}
\ncline{-}{00}{10} \ncput*{$b$}
\ncline{-}{10}{20} \ncput*{$a$}
\ncline{-}{20}{30} \ncput*{$b$}
% \ncline{-}{30}{40}
\ncline{-}{01}{11} \ncput*{$a$}
\ncline{-}{11}{21} \ncput*{$b$}
\ncline{-}{21}{31} \ncput*{$a$}
% \ncline{-}{31}{41}
\endpspicture
}

\bigskip

To determine the model restriction for an announcement, given a number pair $(m,n)$ let $m_\agent$ stand for `Anne is told number $m$' and let $n_\agentb$ stand for `Bill is told number $n$'. Each state corresponds to a number pair $(m,n)$. To process Anne's first announcement we determine if $\neg B_\agent n_\agentb$ in that state. If so, the state is kept. Otherwise, the state is eliminated. For Bill's subsequent announcement we do this for $\neg B_\agentb m_\agent$. Etc.: the third time for $B_\agent n_\agentb$, and the last time for $B_\agentb m_\agent$.\footnote{Although the riddle is a standard illustration of truthful public announcement logic, these announcements cannot actually be formalized in the language of that logic. For the first announcement it would be the infinitary expression $\Et_{n \in \Nat} (n_\agentb \imp \neg B_\agent n_\agentb)$ (the expression holds for all $n \neq 1$). This is not a formula in $\lang(!)$. To demonstrate $\agent$ and $\agentb$ lying to each other this does not matter.} We successively process all four announcements. The epistemic state resulting from an announcement is shown immediately after it.

\begin{itemize}
\item Anne: ``I do not know your number.''
\end{itemize}

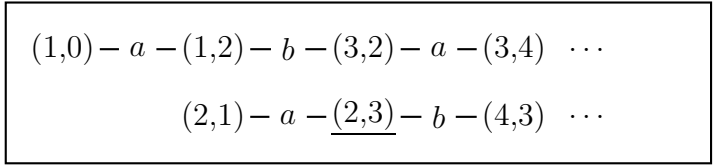
\psframebox{
\psset{border=1pt, nodesep=2pt, radius=2pt, tnpos=a}
\pspicture(-0.5,-0.2)(8.5,1.7)
\rput(2,0.3){\rnode{10}{(2,1)}}
\rput(4,0.3){\rnode{20}{\underline{(2,3)}}}
\rput(6,0.3){\rnode{30}{(4,3)}}
\rput(7,0.3){\rnode{40}{\dots}}
\rput(0,1.2){\rnode{01}{(1,0)}}
\rput(2,1.2){\rnode{11}{(1,2)}}
\rput(4,1.2){\rnode{21}{(3,2)}}
\rput(6,1.2){\rnode{31}{(3,4)}}
\rput(7,1.2){\rnode{41}{\dots}}
\ncline{-}{10}{20} \ncput*{$a$}
\ncline{-}{20}{30} \ncput*{$b$}
% \ncline{-}{30}{40}
\ncline{-}{01}{11} \ncput*{$a$}
\ncline{-}{11}{21} \ncput*{$b$}
\ncline{-}{21}{31} \ncput*{$a$}
% \ncline{-}{31}{41}
\endpspicture
}

\begin{itemize}
\item Bill: ``I do not know your number.''
\end{itemize}

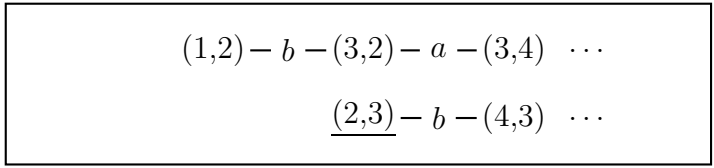
\psframebox{
\psset{border=1pt, nodesep=2pt, radius=2pt, tnpos=a}
\pspicture(-0.5,-0.2)(8.5,1.7)
\rput(4,0.3){\rnode{20}{\underline{(2,3)}}}
\rput(6,0.3){\rnode{30}{(4,3)}}
\rput(7,0.3){\rnode{40}{\dots}}
\rput(2,1.2){\rnode{11}{(1,2)}}
\rput(4,1.2){\rnode{21}{(3,2)}}
\rput(6,1.2){\rnode{31}{(3,4)}}
\rput(7,1.2){\rnode{41}{\dots}}
\ncline{-}{20}{30} \ncput*{$b$}
% \ncline{-}{30}{40}
\ncline{-}{11}{21} \ncput*{$b$}
\ncline{-}{21}{31} \ncput*{$a$}
% \ncline{-}{31}{41}
\endpspicture
}

\begin{itemize}
\item Anne: ``I know your number.''
\end{itemize}

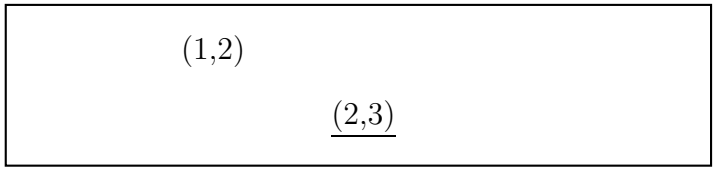
\psframebox{
\psset{border=1pt, nodesep=2pt, radius=2pt, tnpos=a}
\pspicture(-0.5,-0.2)(8.5,1.7)
\rput(4,0.3){\rnode{20}{\underline{(2,3)}}}
\rput(2,1.2){\rnode{11}{(1,2)}}
\endpspicture
}
\begin{itemize}
\item Bill: ``I know your number.''
\end{itemize}
The last announcement does not make a difference anymore, as it is already common knowledge that Anne and Bill know each other's number.

\paragraph*{Consecutive numbers with lying} Next, we show two different scenarios for the consecutive numbers riddle with lying. This is agent truthtelling and agent lying, the epistemic actions defined as $!_\agent \phi$ and $\lie_\agent \phi$ (Definition \ref{def.langagent}). With lying, the riddle involves (the speaker) feigning knowledge and (the addressee) incorrectly believing something to be knowledge, so we move from knowledge to belief. In the communication only occur `I know your number' and `I do not know your number'. Bluffing can therefore not be demonstrated: introspective agents believe their uncertainty and believe their beliefs. As we are reasoning from the actual state $(2,3)$, to simplify matters we do not depict the disconnected upper part of the model. And as beliefs may be incorrect, we show all arrows.

\paragraph*{First scenario} The first scenario consists of Anne lying in her first announcement. Bill does not believe Anne's announcement: his accessibility relation has become empty. Bill's beliefs are therefore no longer consistent. Here, the analysis stops. We do not treat Bill's `announcement' ``That's a lie'' as a permitted move in the game. On the assumption of initial common knowledge Bill {\em knows} that Anne was lying and not mistaken.

\bigskip
\bigskip
\bigskip

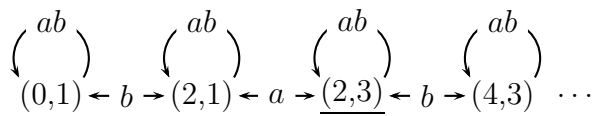
\psset{border=1pt, nodesep=2pt, radius=2pt, tnpos=a}
\pspicture(-0.5,-0.2)(8.5,0.8)
\rput(0,0.3){\rnode{00}{(0,1)}}
\rput(2,0.3){\rnode{10}{(2,1)}}
\rput(4,0.3){\rnode{20}{\underline{(2,3)}}}
\rput(6,0.3){\rnode{30}{(4,3)}}
\rput(7,0.3){\rnode{40}{\dots}}
\ncline{<->}{00}{10} \ncput*{$b$}
\ncline{<->}{10}{20} \ncput*{$a$}
\ncline{<->}{20}{30} \ncput*{$b$}
% \ncline{<->}{30}{40}
\nccircle[angle=0]{->}{00}{.5} \ncput*{$ab$}
\nccircle[angle=0]{->}{10}{.5} \ncput*{$ab$}
\nccircle[angle=0]{->}{20}{.5} \ncput*{$ab$}
\nccircle[angle=0]{->}{30}{.5} \ncput*{$ab$}
\endpspicture

\vspace{-.5cm}

\begin{itemize}
\item Anne: ``I know your number.'' {\it Anne is lying}
\end{itemize}

\bigskip
\bigskip

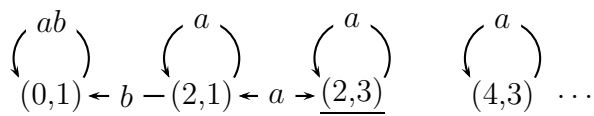
\psset{border=1pt, nodesep=2pt, radius=2pt, tnpos=a}
\pspicture(-0.5,-0.2)(8.5,0.8)
\rput(0,0.3){\rnode{00}{(0,1)}}
\rput(2,0.3){\rnode{10}{(2,1)}}
\rput(4,0.3){\rnode{20}{\underline{(2,3)}}}
\rput(6,0.3){\rnode{30}{(4,3)}}
\rput(7,0.3){\rnode{40}{\dots}}
\ncline{<-}{00}{10} \ncput*{$b$}
\ncline{<->}{10}{20} \ncput*{$a$}
%\ncline{<->}{20}{30} \ncput*{$b$}
% \ncline{<->}{30}{40}
\nccircle[angle=0]{->}{00}{.5} \ncput*{$ab$}
\nccircle[angle=0]{->}{10}{.5} \ncput*{$a$}
\nccircle[angle=0]{->}{20}{.5} \ncput*{$a$}
\nccircle[angle=0]{->}{30}{.5} \ncput*{$a$}
\endpspicture

\vspace{-.5cm}

\begin{itemize}
\item Bill: ``That's a lie.''
\end{itemize}

\paragraph*{Second scenario} In the second scenario Anne initially tells the truth, after which Bill is lying, resulting in Anne mistakenly concluding (and announcing) that she knows Bill's number: she believes it to be 1. This mistaken announcement by Anne is informative to Bill. He learns from it (correctly) that Anne's number is 2, something he didn't know before. 

In the second scenario, Bill gets away with lying, because Anne considered it possible that he told the truth. Bill knows (believes correctly, and justifiably) that Anne's second announcement was a mistake and not a lie. The justification is, again, common knowledge in the initial epistemic state.

\bigskip
\bigskip
\bigskip

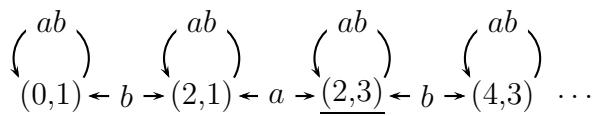
\psset{border=1pt, nodesep=2pt, radius=2pt, tnpos=a}
\pspicture(-0.5,-0.2)(8.5,0.8)
\rput(0,0.3){\rnode{00}{(0,1)}}
\rput(2,0.3){\rnode{10}{(2,1)}}
\rput(4,0.3){\rnode{20}{\underline{(2,3)}}}
\rput(6,0.3){\rnode{30}{(4,3)}}
\rput(7,0.3){\rnode{40}{\dots}}
\ncline{<->}{00}{10} \ncput*{$b$}
\ncline{<->}{10}{20} \ncput*{$a$}
\ncline{<->}{20}{30} \ncput*{$b$}
% \ncline{<->}{30}{40}
\nccircle[angle=0]{->}{00}{.5} \ncput*{$ab$}
\nccircle[angle=0]{->}{10}{.5} \ncput*{$ab$}
\nccircle[angle=0]{->}{20}{.5} \ncput*{$ab$}
\nccircle[angle=0]{->}{30}{.5} \ncput*{$ab$}
\endpspicture

\vspace{-.5cm}

\begin{itemize}
\item Anne: ``I do not know your number.''
\end{itemize}

\bigskip
\bigskip

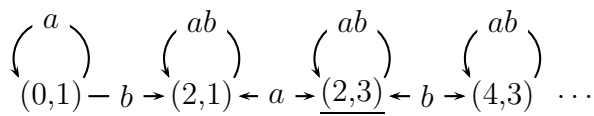
\psset{border=1pt, nodesep=2pt, radius=2pt, tnpos=a}
\pspicture(-0.5,-0.2)(8.5,0.8)
\rput(0,0.3){\rnode{00}{(0,1)}}
\rput(2,0.3){\rnode{10}{(2,1)}}
\rput(4,0.3){\rnode{20}{\underline{(2,3)}}}
\rput(6,0.3){\rnode{30}{(4,3)}}
\rput(7,0.3){\rnode{40}{\dots}}
\ncline{->}{00}{10} \ncput*{$b$}
\ncline{<->}{10}{20} \ncput*{$a$}
\ncline{<->}{20}{30} \ncput*{$b$}
% \ncline{<->}{30}{40}
\nccircle[angle=0]{->}{00}{.5} \ncput*{$a$}
\nccircle[angle=0]{->}{10}{.5} \ncput*{$ab$}
\nccircle[angle=0]{->}{20}{.5} \ncput*{$ab$}
\nccircle[angle=0]{->}{30}{.5} \ncput*{$ab$}
\endpspicture

\vspace{-.5cm}

\begin{itemize}
\item Bill: ``I know your number.'' {\em Bill is lying}
\end{itemize}

\bigskip
\bigskip

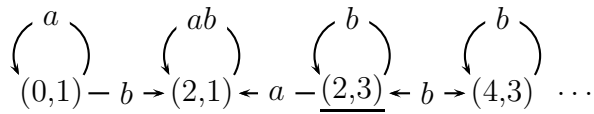
\psset{border=1pt, nodesep=2pt, radius=2pt, tnpos=a}
\pspicture(-0.5,-0.2)(8.5,0.8)
\rput(0,0.3){\rnode{00}{(0,1)}}
\rput(2,0.3){\rnode{10}{(2,1)}}
\rput(4,0.3){\rnode{20}{\underline{(2,3)}}}
\rput(6,0.3){\rnode{30}{(4,3)}}
\rput(7,0.3){\rnode{40}{\dots}}
\ncline{->}{00}{10} \ncput*{$b$}
\ncline{<-}{10}{20} \ncput*{$a$}
\ncline{<->}{20}{30} \ncput*{$b$}
%% \ncline{<->}{30}{40}
\nccircle[angle=0]{->}{00}{.5} \ncput*{$a$}
\nccircle[angle=0]{->}{10}{.5} \ncput*{$ab$}
\nccircle[angle=0]{->}{20}{.5} \ncput*{$b$}
\nccircle[angle=0]{->}{30}{.5} \ncput*{$b$}
\endpspicture

\vspace{-.5cm}

\begin{itemize}
\item Anne: ``I know your number.'' {\em Anne is mistaken.}
\end{itemize}

\bigskip
\bigskip

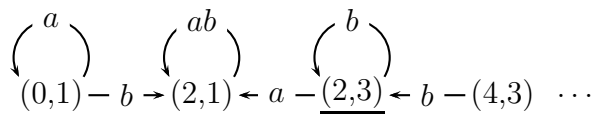
\psset{border=1pt, nodesep=2pt, radius=2pt, tnpos=a}
\pspicture(-0.5,-0.2)(8.5,0.8)
\rput(0,0.3){\rnode{00}{(0,1)}}
\rput(2,0.3){\rnode{10}{(2,1)}}
\rput(4,0.3){\rnode{20}{\underline{(2,3)}}}
\rput(6,0.3){\rnode{30}{(4,3)}}
\rput(7,0.3){\rnode{40}{\dots}}
\ncline{->}{00}{10} \ncput*{$b$}
\ncline{<-}{10}{20} \ncput*{$a$}
\ncline{<-}{20}{30} \ncput*{$b$}
%% \ncline{<->}{30}{40}
\nccircle[angle=0]{->}{00}{.5} \ncput*{$a$}
\nccircle[angle=0]{->}{10}{.5} \ncput*{$ab$}
\nccircle[angle=0]{->}{20}{.5} \ncput*{$b$}
%\nccircle[angle=0]{->}{30}{.5} \ncput*{$b$}
\endpspicture

\bigskip

Once the agents consider it possible that the other is lying, every announcement can be either truthful or lying. This is obvious, because (just as for the semantics of `believed public announcement') the structural transformation induced by agent announcement depends on the model only, and not the actual state. The two scenarios can therefore be executed in any given state of the epistemic model. Whether the announcements are lying or truthtelling depends on the actual state (only).

For example, in the first scenario, if the initial state had been $(4,3)$, it would also have been a lie, and Bill would have similarly detected the lie. If the initial state had been $(2,1)$, it would still have been a lie, but Bill would not have known that. And if the initial state had been $(0,1)$, it would have been the truth. Similar variations can be outlined for the second scenario.

Instead of considering Bill's ``That's a lie'' as terminating the game by observing an illegal move (breach of protocol), we could consider modelling this as an informative (agent) announcement. It amounts to Bill announcing $\F$, i.e., as the truthful agent announcement $!B_\agentb \F$ (`I believe a false statement'). In $(2,3)$, the formula $B_\agentb \F$ is true. In a state where Bill's beliefs are consistent, it is false. The informative consequences are as follows.

\bigskip
\bigskip
\bigskip

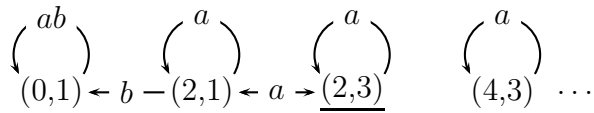
\psset{border=1pt, nodesep=2pt, radius=2pt, tnpos=a}
\pspicture(-0.5,-0.2)(8.5,0.8)
\rput(0,0.3){\rnode{00}{(0,1)}}
\rput(2,0.3){\rnode{10}{(2,1)}}
\rput(4,0.3){\rnode{20}{\underline{(2,3)}}}
\rput(6,0.3){\rnode{30}{(4,3)}}
\rput(7,0.3){\rnode{40}{\dots}}
\ncline{<-}{00}{10} \ncput*{$b$}
\ncline{<->}{10}{20} \ncput*{$a$}
%\ncline{<->}{20}{30} \ncput*{$b$}
% \ncline{<->}{30}{40}
\nccircle[angle=0]{->}{00}{.5} \ncput*{$ab$}
\nccircle[angle=0]{->}{10}{.5} \ncput*{$a$}
\nccircle[angle=0]{->}{20}{.5} \ncput*{$a$}
\nccircle[angle=0]{->}{30}{.5} \ncput*{$a$}
\endpspicture

\vspace{-.5cm}

\begin{itemize}
\item Bill: ``That's a lie.''
\end{itemize}

\bigskip
\bigskip

\psset{border=1pt, nodesep=2pt, radius=2pt, tnpos=a}
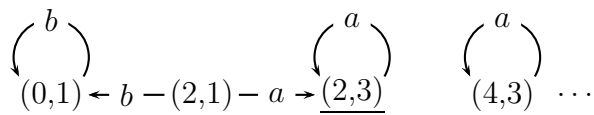
\pspicture(-0.5,-0.2)(8.5,0.8)
\rput(0,0.3){\rnode{00}{(0,1)}}
\rput(2,0.3){\rnode{10}{(2,1)}}
\rput(4,0.3){\rnode{20}{\underline{(2,3)}}}
\rput(6,0.3){\rnode{30}{(4,3)}}
\rput(7,0.3){\rnode{40}{\dots}}
\ncline{<-}{00}{10} \ncput*{$b$}
\ncline{->}{10}{20} \ncput*{$a$}
%\ncline{<->}{20}{30} \ncput*{$b$}
% \ncline{<->}{30}{40}
\nccircle[angle=0]{->}{00}{.5} \ncput*{$b$}
%\nccircle[angle=0]{->}{10}{.5} \ncput*{$a$}
\nccircle[angle=0]{->}{20}{.5} \ncput*{$a$}
\nccircle[angle=0]{->}{30}{.5} \ncput*{$a$}
\endpspicture

\bigskip

\noindent Of course, this way we declare it a permitted move in the consecutive numbers game, and Bill could have lied in his announcement! Consider actual number pair $(0,1)$. Anne truthfully announces ``I know your number.'' Following Bill's response ``That's a lie'', Anne will now even more indignantly contest: ``That cannot be, {\em you} are lying!'', i.e., ``Your announcement ``That's a lie'' was a lie.'' We recall that given initial common knowledge Anne knows that Bill is in fact lying. There is no uncertainty. Lying games may benefit from analysis in game theory, by associating a high negative payoff to the discovery of lying, see Section \ref{sec.further}.

\section{Action models and lying} \label{sec.am}

Whether I am truthfully announcing $\phi$ to you, or am lying that $\phi$, or am bluffing that $\phi$, to you it all appears as the same annoucement $\phi!$. Whereas to me, the speaker, they all appear to be different announcements: I know whether I am truthful, lying, or bluffing. Different agents have different perspectives on this action. Action models \cite{baltagetal:1998} are a familiar way to formalize uncertainty about actions in that form of different perspectives on `what the real action is'. The action model for truthful public announcement (state elimination semantics) can be viewed as a singleton action model. This is well-known. We can view truthful and lying public announcement (`believed announcement', the arrow elimination semantics) as the different points, respectively, of a two-point action model. This is somewhat less well-known. (The modelling of believed announcements \cite{gerbrandyetal:1997} as an action model was suggested in \cite{jfak.jancl:2007,kooi.jancl:2007}.) We can also view truthful, lying and bluffing agent announcement as the respective different points of a three-point action model. These should be seen as alternative descriptions of the logics of lying in terms of a well-known framework. It has the additional advantage of independently validating the various axioms for lying and bluffing, and providing alternative completeness proofs (action model logic is complete).  All the following definitions can be found in any standard introduction to action models, such as \cite{hvdetal.del:2007}.

\bigskip

An action model is a structure like a Kripke model but with a precondition function instead of a valuation function. 
\begin{definition}[Action model]
An {\em action model} $\amodel = ( \Actions, \arel, \pre )$ consists of a {\em domain} $\Actions$ of {\em actions}, an {\em accessibility function} $\arel: \Agents \imp {\mathcal P}(\Actions \times \Actions)$, where each $\arel_\agent$ is an accessibility relation, and a {\em precondition function} $\pre: \Actions \imp \lang_X$, where $\lang_X$ is a logical language. A pointed action model is an {\em epistemic action}. 
\end{definition}
A truthful public announcement of $\phi$ (state elimination semantics) is a singleton action model with precondition $\phi$ and with the single action accessible to all agents. 

Performing an epistemic action in an epistemic state means computing their restricted modal product. This product encodes the new state of information.
\begin{definition}[Update of an epistemic state with an action model] \label{def.updateam}
Given an epistemic state $(M,\state)$ where $M = ( \States, R, V )$ and an epistemic action $(\amodel,\actiona)$ where $\amodel = ( \Actions, \arel, \pre )$. Let $M,\state \models \pre(\actiona)$. The update $(M \otimes \amodel, (\state,\actiona))$ is an epistemic state where $M \otimes \amodel = ( \States', R', V' )$ such that \[ \begin{array}{lcl}  \States' & = & \{ (\stateb,\actionb) \mid M,\stateb \models \pre(\actionb) \} \\ ((\stateb,\actionb),(\stateb',\actionb')) \in R'_a & \text{iff} & (\stateb,\stateb') \in R_a \text{ and } (\actionb,\actionb') \in \arel_a \\ (\stateb,\actionb) \in V'(\atom) & \text{iff} & \stateb \in V(\atom)
\end{array} \]
\end{definition}
The domain of $M \otimes \amodel$ is the product of the domains of $M$ and $\amodel$, but restricted to state/action pairs $(\stateb,\actionb)$ such that $M,\stateb \models \pre(\actionb)$, i.e., such that the action can be executed in that state. (Note that a state in the resulting domain is no longer an abstract object, as before, but such a state/action-pair.) An agent considers a pair $(\stateb,\actionb)$ possible in the next information state if she considered the previous state $\stateb$ possible, and the execution of action $\actionb$ in that state. And the valuations do not change after action execution.

\begin{definition}[Action model for truthful and lying pub.\ ann.] \label{def.actlying}
The action model $\amodel'$ for truthful and lying public announcement that $\phi$, where $\phi\in\lang(!,\lie)$, consists of two actions (suggestively) named $!$ and $\lie$, with $\pre(!)=\phi$ and $\pre(\lie)=\neg\phi$, and such that for all agents only action $!$ is accessible. Truthful public announcement of $\phi$ is the epistemic action $(\amodel',!)$. Lying public announcement of $\phi$ is the epistemic action $(\amodel',\lie)$. The action model can be depicted as follows. Preconditions are shown below the actions.

\bigskip

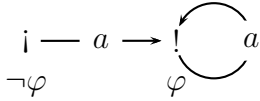
\psset{border=2pt, nodesep=4pt, radius=2pt, tnpos=a}
\pspicture(-1,0)(3.5,0)
$
\rput(0,0){\rnode{00}{\lie}}
\rput(2,0){\rnode{10}{!}}
\rput(0,-.6){\rnode{00a}{\neg\phi}}
\rput(2,-.6){\rnode{10a}{\phi}}
\ncline{->}{00}{10} \ncput*{\agent}
\nccircle[angle=270]{->}{10}{.5} \ncput*{\agent}
$
\endpspicture

\bigskip \ 
\end{definition}
This terminology is not ambiguous in view of our earlier results, as we have the following.
\begin{proposition}[\cite{jfak.jancl:2007,kooi.jancl:2007}] \label{prop.corr}
\[ \begin{array}{lcl}
M, \state \models [!\phi]\psi & \text{iff} & M \otimes \amodel', (\state,!) \models \psi \\
M, \state \models [\lie\phi]\psi & \text{iff} & M \otimes \amodel', (\state,\lie) \models \psi \ .
\end{array} \]
\end{proposition}
\begin{proof}
Elementary, by induction on $\psi$.
\end{proof}

For an example, we show the execution of $(\amodel',\lie)$ for `the outside observer is lying to $\agent$ that $\atom$', in the epistemic state where the agent $\agent$ is uncertain whether $\atom$ but where $\atom$ is false.

\bigskip

\noindent
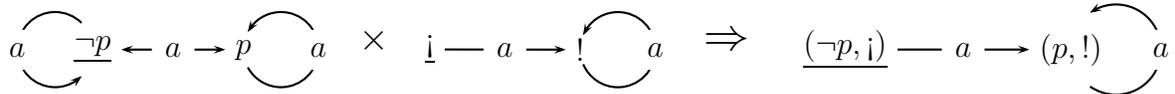
\psset{border=2pt, nodesep=4pt, radius=2pt, tnpos=a}
\pspicture(-1.5,-0)(2,0)
$
\rput(0,0){\rnode{00}{\underline{\neg\atom}}}
\rput(2,0){\rnode{10}{\atom}}
\ncline{<->}{00}{10} \ncput*{\agent}
\nccircle[angle=90]{->}{00}{.5} \ncput*{\agent}
\nccircle[angle=270]{->}{10}{.5} \ncput*{\agent}
$
\endpspicture
\hspace{1.4cm} {\large $\times$} \ \ \ 
\psset{border=2pt, nodesep=4pt, radius=2pt, tnpos=a}
\pspicture(0,0)(2,0)
$
\rput(0,0){\rnode{00}{\underline{\lie}}}
\rput(2,0){\rnode{10}{!}}
\ncline{->}{00}{10} \ncput*{\agent}
\nccircle[angle=270]{->}{10}{.5} \ncput*{\agent}
$
\endpspicture
\hspace{1.5cm} {\Large $\Imp$} \ 
\psset{border=2pt, nodesep=4pt, radius=2pt, tnpos=a}
\pspicture(-1,0)(2,0)
$
\rput(0,0){\rnode{00}{\underline{(\neg\atom,\lie)}}}
\rput(3,0){\rnode{10}{(\atom,!)}}
\ncline{->}{00}{10} \ncput*{\agent}
\nccircle[angle=270]{->}{10}{.6} \ncput*{\agent}
$
\endpspicture

\bigskip
\bigskip

Consider this epistemic model $( \States, R, V )$. As before, the states in the epistemic model are named after their valuation, so that $\States = \{ \neg\atom, \atom \}$; the state named $\neg\atom$ is simply the state where $\atom$ is false, etc. The modal product $( \States', R', V' )$ consists of two states; $(\neg\atom,\lie) \in \States'$ because $M,\neg\atom \models \pre(\lie)$, as $\pre(\lie) = \neg \atom$, and $(\atom,!) \in \States'$ because  $M,\atom \models \atom$. Then, $((\neg\atom,\lie), (\atom,!)) \in R'_\agent$ because $(\neg\atom,\atom) \in R_\agent$ (in $M$) and $(\lie,!) \in \arel_\agent$ (in the action model $\amodel'$), etc. It is an artifact of the example that the shape of the action model is the shape of the next epistemic state. That is merely a consequence of the fact that the initial epistemic state has the universal accessibility relation for the agent on a domain of all valuations of the atoms occurring in the precondition (here: a domain of two states, for the two valuations of $\atom$). 

%\paragraph*{Action model logic} 
Possibly more elegantly, we can also show the correspondence established in Proposition \ref{prop.corr} from the perspective of a different logic. With an epistemic action $(\amodel,\actiona)$ we can associate a dynamic modal operator $[\amodel,\actiona]$ in a logical language where an enumeration of action model frames is a parameter of the inductive language definition, apart from propositional variables $\Atoms$ and agents $\Agents$. 
\begin{definition}[Language of action model logic]
\[ \lang(\otimes) \ \ni \ \phi ::= \atom \ | \ \neg \phi \ | \ (\phi \et \psi) \ | \ B_\agent \phi \ | \ [\amodel,\actiona] \psi \] 
\end{definition} The last clause is in fact inductive, if we realize that the preconditions of all actions in $\amodel$, including $\actiona$, are also of type formula. For example, truthful public announcement logic is an instantiation of that language for the singleton set of actions $\{!\}$, where we view `$!$' as an operation with two input formulae $\phi$ and $\psi$ and that returns as output the announcement formula $[!\phi]\psi$.  
\begin{definition}[Semantics of ${[}\amodel,\actiona{]}$]
\[ M,\state\models [\amodel,\actiona] \psi \ \text{iff} \ M,\state\models \pre(\actiona) \text{ implies } M \otimes \amodel, (\state,\actiona) \models \psi \]
\end{definition}
Given the action model for truthful and lying public announcement that $\phi$ of Definition \ref{def.actlying}, but with $\phi\in\lang(\otimes)$, and given an inductively defined translation $tr: \lang(\otimes) \imp \lang(!,\lie)$ with only nontrivial clauses $tr([\amodel',!]\psi) := [!tr(\phi)]tr(\psi)$ and $tr([\amodel',\lie]\psi) := [\lie tr(\phi)]tr(\psi)$, the correspondence established in Proposition \ref{prop.corr} can be viewed from the perspective of the dynamic modal operators for action models (where we suggestively execute the first step of the translation).
\[ \begin{array}{lcl}
M, \state \models [!tr(\phi)]tr(\psi) & \text{iff} & M,\state\models [\amodel',!]\psi \\
M, \state \models [\lie tr(\phi)]tr(\psi) & \text{iff} & M,\state\models [\amodel',\lie] \psi 
\end{array} \]

\bigskip

%\paragraph*{Action models for agent announcements} 
We now proceed with the presentation of action models for agent announcements.

\begin{definition}[Action model for agent announcement]
The action model $\amodel''$ for agent announcement consists of three actions named $\bluff_\agent$, $!_\agent$, and $\lie_\agent$ with preconditions $\neg (B_\agent \phi \vel B_\agent \neg \phi)$, $B_\agent \phi$, and $B_\agent\neg\phi$, respectively, where $\phi \in \lang(!_\agent,\lie_\agent,\bluff_\agent)$. The announcing agent $\agent$ has identity access on the action model. To the other agents $\agentb$ only action $!_\agent$ is accessible. Agent $\agent$ truthfully announcing $\phi$ to all other $\agentb$ is is the epistemic action $(\amodel'',!_\agent)$ --- with precondition $B_\agent \phi$, therefore --- and similarly lying and bluffing are the action models $(\amodel'',\lie_\agent)$ and $(\amodel'', \bluff_\agent)$. 

\bigskip
\bigskip
\bigskip

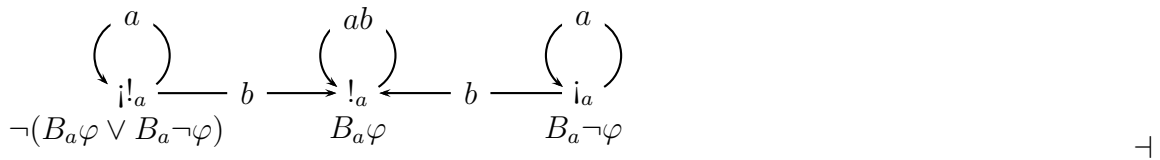
\psset{border=2pt, nodesep=4pt, radius=2pt, tnpos=a}
\pspicture(-3,0)(5,0)
$
\rput(-1,0){\rnode{00}{\bluff_\agent}}
\rput(2,0){\rnode{10}{!_\agent}}
\rput(5,0){\rnode{20}{\lie_\agent}}
\rput(-1,-.5){\rnode{00b}{\neg (B_\agent \phi \vel B_\agent \neg \phi) \ \ \ }}
\rput(2,-.5){\rnode{10b}{B_\agent \phi}}
\rput(5,-.5){\rnode{20b}{B_\agent \neg \phi}}
\ncline{->}{00}{10} \ncput*{\agentb}
\ncline{<-}{10}{20} \ncput*{\agentb}
\nccircle{->}{00}{.5} \ncput*{\agent}
\nccircle{->}{10}{.5} \ncput*{\agent\agentb}
\nccircle{->}{20}{.5} \ncput*{\agent}
$
\endpspicture

\bigskip \ 
\end{definition}

\noindent Again, we have the desired correspondence (and this can, again, be formulated in action model logic with an inductive translation).
\begin{proposition}
\[ \begin{array}{lcl}
M, \state \models [!_\agent \phi]\psi & \text{iff} & M \otimes \amodel'', (\state,!_\agent) \models \psi \\
M, \state \models [\lie_\agent \phi]\psi & \text{iff} & M \otimes \amodel'', (\state,\lie_\agent) \models \psi \\
M, \state \models [\bluff_\agent \phi]\psi & \text{iff} & M \otimes \amodel'', (\state,\bluff_\agent) \models \psi \ 
\end{array} \]
\end{proposition}
The action model representations validate the axioms for announcement and belief, for all versions shown; and they justify that these axioms form part of complete axiomatizations. These axioms are instantiations of the more general axiom for an epistemic action followed by a belief, in action model logic. This axiom (\cite{baltagetal:1998}) is \[ [\amodel,\actiona] B_\agent \psi \ \ \eq \ \ \pre(\actiona) \imp \Et_{(\actiona,\actionb) \in \arel_\agent} B_\agent [\amodel,\actionb] \psi  \] In other words, an agent believes $\psi$ after a given action, if $\psi$ holds after any action that is for $\agent$ indistinguishable from it. For example, in the epistemic action $(\amodel',\lie)$, with $\pre(\lie) = \phi$, for lying public announcement that $\phi$, $!$ is the only accessible action from action $\lie$, and we get \[ [\amodel',\lie] B_\agent \psi \ \ \eq \ \ \pre(\lie) \imp B_\agent [\amodel',!] \psi \] and therefore (we recall Definition \ref{def.axiombellie}) \[ [\lie\phi] B_\agent \psi \ \ \eq \ \ \neg\phi \imp B_\agent [!\phi] \psi \] 

In view of the identification discussed in this section, in the following we will continue to call any dynamic modal operator for belief change an epistemic action, and also standardly represent epistemic actions by their corresponding action models.

\section{Unbelievable lies and skeptical agents} \label{sec.unbel}

If I tell you $\phi$ and you already believe the opposite, accepting this information will make your beliefs inconsistent. This is not merely a problem for lying (`unbelievable lies') but for any form of information update. One way to preserve consistent beliefs is to reject new information if it is inconsistent with your beliefs. (The other way is to accept them, but to remove inconsistent prior beliefs. See Section \ref{sec.plaus}.) Such agents may be called skeptical. In this section we adjust the logics of truthful and lying public announcement and of agent announcement to sceptical agents. This adjustment is elementary. As this topic is alive in the community of dynamic epistemic logicians we incorporate a review of the literature.

Consistency of beliefs is preserved iff seriality is preserved on epistemic models to interpret these beliefs. For the model classes ${\mathcal KD45}$ and ${\mathcal S5}$ that we target, this is the requirement that the class is closed under information update. The perspective on epistemic actions and action models from Section \ref{sec.am} is instructive. The class of ${\mathcal S5}$ epistemic models is closed under update with ${\mathcal S5}$ epistemic actions, such as truthful public announcements, but the class of ${\mathcal KD45}$ models is {\em not} closed under update with ${\mathcal KD45}$ epistemic actions (see the last paragraph of Section \ref{sec.prelim}). The action models for truthful and lying public announcement, and for agent announcement, are ${\mathcal KD45}$.

Updates that preserve ${\mathcal KD45}$ have been investigated in \cite{steiner:2006,aucher:2008,kooietal:2011}. Aucher \cite{aucher:2008} defines a language fragment that makes you go mad (`crazy formulas'). The idea is then to avoid that. Steiner \cite{steiner:2006} proposes that the agent does not incorporate the new information if it already believes to the contrary. In that case, nothing happens. Otherwise, access to states where the information is not believed is eliminated, just as for believed public announcements.\footnote{Steiner gives an interesting parable for the case where you do not accept new information. Someone is calling you and is telling you something that you don't want to believe. What do you do? You start shouting through the phone: `What did you say? Is there anyone on the other side? The connection is bad!' And then you hang up, quickly, before the caller can repeat his message. Thus you create common knowledge (by convention) that the message has been received but its content not accepted.

This analysis comes close but is just off, because to the caller this is indistinguishable from the case where the message is believable to the addressee and would have been accepted. In the context of \cite{steiner:2006} the analysis appears justifiable because the caller is the outside observer who is not modelled. But the outside observer {\em knows} whether the receiver will find the message believable (Proposition \ref{prop.gd}), unlike in Steiner's scenario.} This solution to model unbelievable lies (and unbelievable truths!) is similarly proposed in the elegant \cite{kooietal:2011}, where it is called `cautious update' --- a suitable term. We will propose to call such agents {\em skeptical} instead of cautious.

\subsection{Logic of truthful and lying public announcements to skeptical agents}

We propose a three-point action model for {\em truthful and lying public announcement to skeptical agents}, with the semantics motivated by \cite{steiner:2006,kooietal:2011}.

\begin{definition}[Public announcement to skeptical agents] The action model $\amodelb$ for truthful and lying public announcement to skeptical agents consists of three actions named $!^\scep$, $\lie^\scep$, and (for lack of a better symbol) $!!^\scep$, with preconditions and accessibility relations (for all agents $\agent$) as follows.

%\bigskip

\psset{border=2pt, nodesep=4pt, radius=2pt, tnpos=a}
\pspicture(-2,-.7)(6.5,1.5)
$
\rput(-1,0){\rnode{00}{!!^\scep}}
\rput(2,0){\rnode{10}{\lie^\scep}}
\rput(5,0){\rnode{20}{!^\scep}}
\rput(-1,-.5){\rnode{00b}{B_\agent \neg \phi}}
\rput(2,-.5){\rnode{10b}{\neg\phi\et\neg B_\agent \neg\phi}}
\rput(5,-.5){\rnode{20b}{\phi\et \neg B_\agent \neg \phi}}
\ncline{->}{10}{20} \ncput*{\agent}
\nccircle{->}{00}{.5} \ncput*{\agent}
%\nccircle{->}{10}{.5} \ncput*{\agent}
\nccircle{->}{20}{.5} \ncput*{\agent}
$
\endpspicture

\bigskip

\noindent For $(\amodelb,!^\scep)$, with $\pre(!^\scep) = \phi\et\neg B_\agent \neg \phi$, we write $!^\scep\phi$; similarly, we write $\lie^\scep\phi$ for $(\amodelb,\lie^\scep)$ and $!!^\scep\phi$ for $(\amodelb,!!^\scep)$. 
\end{definition}
The difference with the action model for truthful and lying public announcement is that the alternatives $\phi$ and $\neg\phi$ now have an additional precondition $\neg B_\agent \neg \phi$: the announcement should be believable. In the action model there is a separate, disconnected, case for unbelievable announcements: precondition $B_\agent \neg \phi$. For unbelievable announcements it does not matter whether $\phi$ is a lie or is the truth. The agent chooses to discard it either way. It is skeptical.

\begin{definition}[Axioms]
The principles for public announcements to skeptical agents are:
\[ \begin{array}{lcl} 
\mbox{} [!^\scep \phi] B_\agent \psi &  \leftrightarrow & (\phi \et \neg B_\agent \neg \phi) \rightarrow B_\agent [!^\scep \phi] \psi \\
\mbox{} [\lie^\scep \phi] B_\agent \psi &  \leftrightarrow & (\neg\phi \et \neg B_\agent \neg \phi) \rightarrow B_\agent [!^\scep \phi] \psi \\
\mbox{} [!!^\scep \phi] B_\agent \psi &  \leftrightarrow & B_\agent \neg \phi \rightarrow B_\agent \psi
\end{array} \]
\end{definition}

\begin{proposition}
The axiomatization of the logic for public announcements to skeptical agents is complete.
\end{proposition}
\begin{proof}
Directly, from the embedding in action model logic.
\end{proof}

In the case of a single agent, the semantics of an unbelievable announcement of $\phi$ leaves the structure of the unbelievable part of the model unchanged. This means that agent $\agent$ does not change its beliefs. An unbelievable public announcement does not make an informative difference for the skeptical agent. 

If there are more agents, an announcement can be believable for one agent and unbelievable for another agent. The current action model is no longer appropriate. For example, for two agents there are seven distinct actions, namely all Boolean combinations of $\phi$, $B_\agent \neg \phi$, and $B_\agentb \neg \phi$, minus one: the cases $\phi \et B_\agent \neg \phi \et B_\agentb \neg \phi$ and $\neg \phi \et B_\agent \neg \phi \et B_\agentb \neg \phi$ can be merged into a single action with precondition $B_\agent \neg \phi \et B_\agentb \neg \phi$, because they are indistinguishable for $\agent$ and $\agentb$ jointly. Such a two-agent skeptical action model is not a refinement of the action model above. Unlike the above, it is connected. In particular, the axiom $[!!^\scep \phi] B_\agent \psi \ \leftrightarrow \  B_\agent \neg \phi \rightarrow B_\agent \psi$ is no longer valid, and even when $\agent$ believes to the contrary, and therefore does not accept $\phi$, its beliefs may still change in another way. For example of such changing $B_\agent \psi$, if the announcement is that $\atom$ and I believe that $\neg \atom$, but I consider it possible that you are uncertain about $\atom$, then after the announcement I consider it possible that you now believe $\atom$. So for $B_\agent \psi = B_\agent \neg B_\agent B_\agentb \atom$, this was true before and false after the announcement.

% God and the devil already know the informative consequences of their announcements, whether truthful or lying. If they make an unbelievable announcement, they already know that it will not result in any belief change. The multi-agent context wherein some may and others may not believe it makes more sense.

We consider such matters mere variations, and move on. 

\subsection{Logic of agent announcements to skeptical agents}

The analysis becomes more interesting for agent announcements, namely when the speaker is uncertain whether her lie will be believed by the addressee, as in the consecutive numbers riddle. Even if the addressee already believes $\neg\atom$, he may consider it possible that the speaker is truthfully announcing $\atom$ and is not lying. The addressee $\agentb$ then merely concludes that the speaker $\agent$ must be mistaken in her truthful announcement of $\atom$, and thus believes $\neg\atom\et B_\agent\atom$. If $\agent$ was lying, this belief is mistaken. The charitable addressee discards the option that the speaker was lying, with precondition $B_\agent\neg\atom$. 
% Of course the addressee $\agentb$ will {\em still} go mad if he believed $\neg B_\agent\atom$. Then we can again resort to the cautious update solution; but that would be a weirdly strong form for $\agentb$ to stick to his beliefs: $\agentb$ believes that $\agent$ does not believe $\atom$, even when confronted with $\agent$ announcing $\atom$. (Bernard: ``I hate you! You're terrible!'' Andrea: ``No, look at your inner self, you don't really believe that yourself. You like me! You love me! And I love you.'')

\begin{definition}[Action model for agent announcement to skeptics] \label{def.actaskep}
The action model $\amodelb'$ for agent announcements from speaker $\agent$ to skeptical adressee(s) $\agentb$ is as follows. Assume transitivity of accessibility relations.

\bigskip
\bigskip

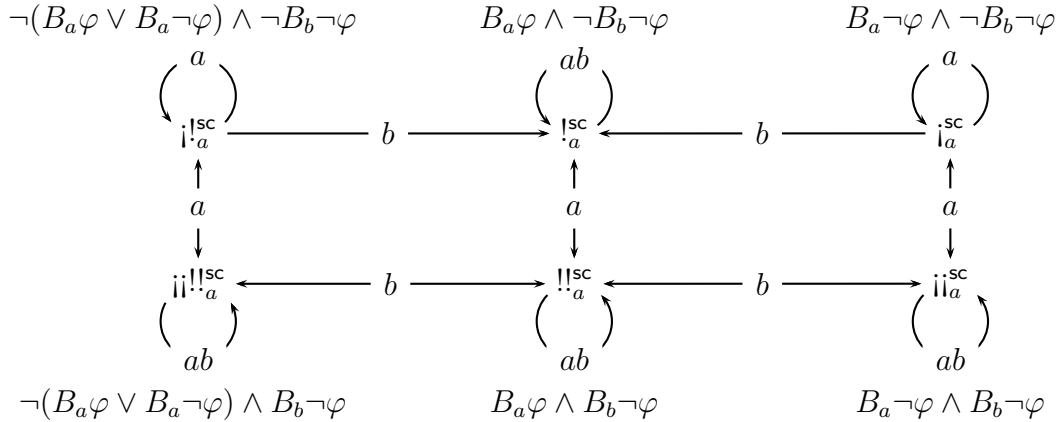
\psset{border=2pt, nodesep=4pt, radius=2pt, tnpos=a}
\pspicture(-3,-3.8)(9,1.5)
$
\rput(-1,0){\rnode{00}{\bluff_\agent^\scep}}
\rput(4,0){\rnode{10}{!_\agent^\scep}}
\rput(9,0){\rnode{20}{\lie_\agent^\scep}}
\rput(-1,1.5){\rnode{00b}{\neg (B_\agent \phi \vel B_\agent \neg \phi) \et \neg B_\agentb \neg \phi \ \ \ }}
\rput(4,1.5){\rnode{10b}{B_\agent \phi \et \neg B_\agentb \neg \phi}}
\rput(9,1.5){\rnode{20b}{B_\agent \neg \phi \et \neg B_\agentb \neg \phi}}
\rput(-1,-2){\rnode{00c}{\lie\lie!!_\agent^\scep}}
\rput(4,-2){\rnode{10c}{!!_\agent^\scep}}
\rput(9,-2){\rnode{20c}{\lie\lie_\agent^\scep}}
\rput(-1,-3.6){\rnode{00bc}{\neg (B_\agent \phi \vel B_\agent \neg \phi) \et B_\agentb \neg \phi \ \ \ }}
\rput(4,-3.6){\rnode{10bc}{B_\agent \phi \et B_\agentb \neg \phi}}
\rput(9,-3.6){\rnode{20bc}{B_\agent \neg \phi \et B_\agentb \neg \phi}}
\ncline{->}{00}{10} \ncput*{\agentb}
\ncline{<-}{10}{20} \ncput*{\agentb}
\ncline{<->}{00c}{10c} \ncput*{\agentb}
\ncline{<->}{10c}{20c} \ncput*{\agentb}
\ncline{<->}{00}{00c} \ncput*{\agent}
\ncline{<->}{10}{10c} \ncput*{\agent}
\ncline{<->}{20}{20c} \ncput*{\agent}
\nccircle{->}{00}{.5} \ncput*{\agent}
\nccircle{->}{10}{.5} \ncput*{\agent\agentb}
\nccircle{->}{20}{.5} \ncput*{\agent}
\nccircle[angle=180]{->}{00c}{.5} \ncput*{\agent\agentb}
\nccircle[angle=180]{->}{10c}{.5} \ncput*{\agent\agentb}
\nccircle[angle=180]{->}{20c}{.5} \ncput*{\agent\agentb}
$
\endpspicture

\bigskip

\noindent For $(\amodelb',\lie_\agent^\scep)$ with precondition $B_\agent\neg\phi\et\neg B_\agentb \neg \phi$ we write $\lie_\agent^\scep\phi$, etc. (Similarly for public announcement to skeptical agents.)
\end{definition}
This action model encodes that $\agent$ knows/believes if she's bluffing, lying, or truthful, but is uncertain if her announcement is affecting $\agentb$'s beliefs: she cannot (from the appearance of the action itself, discarding prior knowledge) distinguish between the actions with preconditions $B_\agentb \neg \phi$ and $\neg B_\agentb \neg \phi$. In case agent $\agentb$ already believed $\neg\phi$, he is indifferent between the three alternatives truthtelling, lying, and bluffing that $\phi$. But in case $\agentb$ considered $\phi$ possible, as before, he believes $\agent$ to be truthful about $\phi$.

We only give the reduction axioms for lying, for this logic, and not the remaining axioms. The more interesting lying axiom is the second one. If formalizes that the lying agent $\agent$ may be uncertain if the lie was believable to the addressee $\agentb$ (the two cases on the right), even if in fact (on the left) the lie was believable. Obviously, the axiomatization is again complete.
\begin{definition}[Axioms] The axioms for $\agent$ lying to a skeptical agent $\agentb$, in case $\agent$ is believed, are as follows. \[ \begin{array}{lll} {[}\lie_\agent^\scep\phi]B_\agentb \psi & \eq & (B_\agent\neg\phi\et\neg B_\agentb \neg \phi) \imp B_\agentb [!_\agent^\scep\phi]\psi \\ {[}\lie_\agent^\scep\phi]B_\agent \psi & \eq & (B_\agent\neg\phi\et\neg B_\agentb \neg \phi) \imp (B_\agent [\lie_\agent^\scep\phi]\psi \et B_\agent [\lie\lie_\agent^\scep\phi]\psi) \end{array} \]

\end{definition}
\begin{proposition}
The axiomatization for agent announcement to skeptical agents is complete.
\end{proposition}

The logic of agent announcements to skeptical agents can be applied to model the consecutive numbers riddle. The reader can compare the below to the first announcement in the first scenario in Section \ref{sec.example}. We observe that both models are ${\mathcal KD45}$.

\bigskip
\bigskip
\bigskip

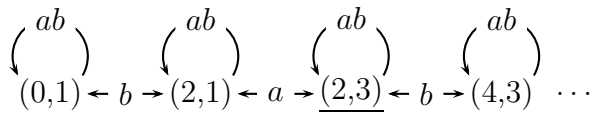
\psset{border=1pt, nodesep=2pt, radius=2pt, tnpos=a}
\pspicture(-0.5,-0.2)(8.5,0.8)
\rput(0,0.3){\rnode{00}{(0,1)}}
\rput(2,0.3){\rnode{10}{(2,1)}}
\rput(4,0.3){\rnode{20}{\underline{(2,3)}}}
\rput(6,0.3){\rnode{30}{(4,3)}}
\rput(7,0.3){\rnode{40}{\dots}}
\ncline{<->}{00}{10} \ncput*{$b$}
\ncline{<->}{10}{20} \ncput*{$a$}
\ncline{<->}{20}{30} \ncput*{$b$}
% \ncline{<->}{30}{40}
\nccircle[angle=0]{->}{00}{.5} \ncput*{$ab$}
\nccircle[angle=0]{->}{10}{.5} \ncput*{$ab$}
\nccircle[angle=0]{->}{20}{.5} \ncput*{$ab$}
\nccircle[angle=0]{->}{30}{.5} \ncput*{$ab$}
\endpspicture

\vspace{-.5cm}

\begin{itemize}
\item Anne: ``I know your number.'' {\it Anne is lying}
\end{itemize}

\bigskip
\bigskip

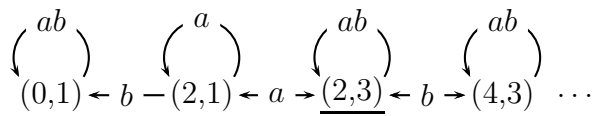
\psset{border=1pt, nodesep=2pt, radius=2pt, tnpos=a}
\pspicture(-0.5,-0.2)(8.5,0.8)
\rput(0,0.3){\rnode{00}{(0,1)}}
\rput(2,0.3){\rnode{10}{(2,1)}}
\rput(4,0.3){\rnode{20}{\underline{(2,3)}}}
\rput(6,0.3){\rnode{30}{(4,3)}}
\rput(7,0.3){\rnode{40}{\dots}}
\ncline{<-}{00}{10} \ncput*{$b$}
\ncline{<->}{10}{20} \ncput*{$a$}
\ncline{<->}{20}{30} \ncput*{$b$}
% \ncline{<->}{30}{40}
\nccircle[angle=0]{->}{00}{.5} \ncput*{$ab$}
\nccircle[angle=0]{->}{10}{.5} \ncput*{$a$}
\nccircle[angle=0]{->}{20}{.5} \ncput*{$ab$}
\nccircle[angle=0]{->}{30}{.5} \ncput*{$ab$}
\endpspicture

\weg{

\vspace{-.5cm}

\begin{itemize}
\item Bill: ``That's a lie.''
\end{itemize}

\bigskip
\bigskip

\psset{border=1pt, nodesep=2pt, radius=2pt, tnpos=a}
\pspicture(-0.5,-0.2)(8.5,0.8)
\rput(0,0.3){\rnode{00}{(0,1)}}
\rput(2,0.3){\rnode{10}{(2,1)}}
\rput(4,0.3){\rnode{20}{\underline{(2,3)}}}
\rput(6,0.3){\rnode{30}{(4,3)}}
\rput(7,0.3){\rnode{40}{\dots}}
\ncline{<-}{00}{10} \ncput*{$b$}
\ncline{->}{10}{20} \ncput*{$a$}
\ncline{<->}{20}{30} \ncput*{$b$}
% \ncline{<->}{30}{40}
\nccircle[angle=0]{->}{00}{.5} \ncput*{$ab$}
%\nccircle[angle=0]{->}{10}{.5} \ncput*{$a$}
\nccircle[angle=0]{->}{20}{.5} \ncput*{$ab$}
\nccircle[angle=0]{->}{30}{.5} \ncput*{$ab$}
\endpspicture

\bigskip
}

\noindent The skeptical Bill continues to be uncertain about Anne's number. He does not change his factual beliefs. But he does change his beliefs about Anne's beliefs. For example, after Anne's lie, Bill considers it possible that Anne considers it possible that he believed her. Informally, this gives Bill reason to believe that Anne has 2 and not 4. If she had 4, she would know that her lie would not be believed.

\section{Lying and plausible belief} \label{sec.plaus}

We recall the three different attitudes, presented in the introductory Section \ref{sec.intro}, towards incorporating announcements $\phi$ that contradict the beliefs $B_\agentb \neg\phi$ of an addressee $\agentb$: (i) do it at the price of inconsistent beliefs (public announcements and agent announcements as treated in Sections \ref{sec.pub} and \ref{sec.agent}), (ii) reject the information (announcements to skeptical agents, treated in Section \ref{sec.unbel}), and (iii) accept the information by a consistency preserving process removing some old beliefs. This section is devoted to the third way.
Going mad is too strong a response, not ever accepting new information seems too weak a response, we now discuss a solution in between. It involves distinguishing stronger from weaker beliefs, when revising beliefs. To achieve that, we need to give epistemic models more structure: given a set of states all considered possible by an agent, it may consider some more plausible than others, and belief in $\phi$ can then be defined as the truth of $\phi$ in the most plausible states that are considered possible. We now have more options to change beliefs. We can change the sets of states considered possible by the agent, but we can also change the relative plausibility of states within that set.

Such approaches for belief change involving plausibility have been proposed in \cite{aucher:2005a,hvd.prolegomena:2005,jfak.jancl:2007,baltagetal.tlg3:2008}. How to model lying with plausibility models is summarily discussed in \cite{baltagetal.tlg:2008,hvd.comments:2008} (as a dialogue, these are different contributions to the same volume), and also in \cite[p.54]{baltagetal.tlg3:2008}. 

We continue in the same vein as in the previous sections and present epistemic actions for plausible (truthful and lying) public announcement, and for plausible (truthful, lying, and bluffing) agent announcement, by an adjustment of the epistemic actions already shown. The epistemic action for plausible public lying is the one in \cite{baltagetal.tlg:2008,hvd.comments:2008,baltagetal.tlg3:2008}, the epistemic action for plausible agent lying applies the general setup of \cite{baltagetal.tlg3:2008} to a specific (plausibility) action model. Thus it appears we can again present reduction axioms for belief change and complete axiomatizations for these logics. This is so, but it would involve not just modal belief operators $B_\agent \phi$ but also conditional belief operators $B_\agent^\phi \psi$ (if $\phi$ were true, agent $\agent$ would believe $\psi$; so that $B_\agent \psi$ is $B_\agent^\T \psi$), and additional axioms for conditional belief. We refer to the cited works by Baltag and Smets \cite{baltagetal.tlg:2008,baltagetal.tlg3:2008} for further details.

\begin{definition}[Plausibility epistemic models] \label{def.pem}
A {\em plausibility epistemic model} $M = ( \States, \sim, {<}, V )$ has one more parameter than an epistemic model, namely a plausibility function ${<}: \Agents \imp \powerset(\States \times \States)$. The accessibility relations for each agent are required to be equivalence relations $\sim_\agent$. The restriction of $<_\agent$ to an equivalence class of $\sim_\agent$ is required to be a {\em prewellorder}\footnote{A total, transitive, and well-founded relation}. 
\end{definition}
As the accessibility relations are equivalence relations, we write $\sim_\agent$ instead of $R_\agent$, as before. If $\state <_\agent \stateb$ we say that state $\state$ is considered more plausible than state $\stateb$ by agent $\agent$. The transitive clause of $<_\agent$ is $\leq_\agent$ and for ($\state \leq_\agent \stateb$ and $\stateb \leq_\agent \state$) we write $\state =_\agent \stateb$ ({\em equally plausible}).

In this setting we can distinguish knowledge from belief: the agent believes $\phi$, iff $\phi$ is true in all most plausible equivalent states; and the agent knows $\phi$, iff $\phi$ is true in all equivalent states.\footnote{The former represents weak belief and the latter true strong belief. There are yet other epistemic operators in this setting, safe belief, conditional belief, ... We restrict our presentation to (weak) belief.} We write $B_\agent\phi$ for `agent $\agent$ believes $\phi$', as before. There is a natural way to associate an accessibility relation with such belief, based on the equivalence relations and plausibility relations, and we can then define belief in $\phi$ as the truth of $\phi$ in all accessible states, as usual. Defined in this way, these relations for belief satisfy the ${\mathcal KD45}$ properties.
\begin{definition}[Accessibility relation and semantics for belief]
Given an plausibility epistemic model $M$ with $\sim_\agent$ and $<_\agent$, the accessibility relation $R_\agent$ is defined as: \[ (\state, \stateb) \in R_\agent \ \text{ iff } \ \state \sim_\agent \stateb \text{ and } \stateb \leq_\agent \stateb' \text{ for all } \stateb' \text{ such that } \state \sim_\agent \stateb' \]
Given a plausibility epistemic model $M$ and a state $\state$ in its domain, $M,\state\models B_\agent \phi$ iff $M,\stateb\models \phi$ for all $\stateb$ such that $(\state,\stateb)\in R_\agent$.  
\end{definition}

An information update that changes the knowledge of the agents, by way of changing the relations $\sim_\agent$, may also affect their beliefs, by way of changing the derived relations $R_\agent$. For example, suppose an agent $\agent$ is uncertain whether $\atom$ but considers it more likely that $\atom$ is true than that $\atom$ is false. Without reason, because in fact $\atom$ is false. The epistemic state is depicted below, on the left. The arrows in the equivalence relation are dashed lines. The arrows in the accessibility relation $R_\agent$ for belief are solid lines, as before. In this state $B_\agent \atom$ is true, because $\atom$ is true in the more plausible state. Now agent $\agent$ is presented with hard evidence that $\neg\atom$ (state elimination semantics, as in truthful public announcement logic). The state where $\atom$ is false is eliminated from consideration. The only remaining state has become the most plausible state. In the epistemic state depicted below on the right, $B_\agent\neg\atom$ is true. Agent $\agent$ has revised her belief that $\atom$ into belief that $\neg\atom$. In this example, belief changes but knowledge also changes. There are other examples wherein only belief changes.

\bigskip
%\bigskip

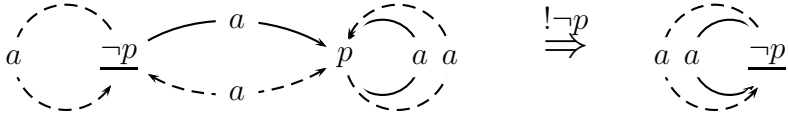
\psset{border=2pt, nodesep=4pt, radius=2pt, tnpos=a}
\pspicture(-3,0)(2.5,0)
$
%\rput(-1,-.4){\rnode{00c}{ \ ~ \ }}
%\rput(2,-.4){\rnode{10c}{ \ ~ \ }}
\rput(-1,0){\rnode{00}{ \underline{\neg\atom} }}
\rput(2,0){\rnode{10}{ \atom }}
\ncarc[arcangle=30]{->}{00}{10} \ncput*{\agent}
\ncarc[linestyle=dashed,arcangle=-30]{<->}{00}{10} \ncput*{\agent}
%\nccircle[angle=90]{->}{00}{.5} \ncput*{\agentb}
\nccircle[angle=270]{->}{10}{.5} \ncput*{\agent}
%\rput(-1,-2){\rnode{00bbl}{1}}
%\rput(2,-2){\rnode{10bbl}{0}}
\nccircle[linestyle=dashed,angle=90]{->}{00}{.7} \ncput*{\agent}
\nccircle[linestyle=dashed,angle=270]{->}{10}{.7} \ncput*{\agent}
$
\endpspicture
\hspace{2cm} {\Large $\stackrel {! \neg \atom} \Imp$} \hspace{.1cm} 
\psset{border=2pt, nodesep=5pt, radius=2pt, tnpos=a}
\pspicture(-3,0)(2.5,0)
$
\rput(-1,-.4){\rnode{00c}{ \ ~ \ }}
\rput(-1,0){\rnode{00}{ \underline{\neg\atom} }}
\nccircle[angle=90]{->}{00}{.5} \ncput*{\agent}
%\rput(-1,-2){\rnode{00bbl}{0}}
\nccircle[linestyle=dashed,angle=90]{->}{00}{.7} \ncput*{\agent}
$
\endpspicture

\bigskip
\bigskip
%\bigskip

\noindent Belief revision consists of changing the plausibility order between states. This induces an order between deductively closed sets of formulas. The most plausible of these is the set of formulas that are believed. This belief revision is similar to AGM belief revision, seen as changing the plausibility order (partial or total order, or prewellorder) between deductively closed sets of formulas. The contraction that forms part of the revision is with respect to that order. Dynamic epistemic logics for belief revision were developed to model higher order belief revision and iterated belief revision. 

Just as epistemic actions, with underlying action models, generalize public announcements, plausibility epistemic actions generalize simpler forms of public plausibility updates. Accessibility relations for `considering an action possible' are computed as in the case of plausibility epistemic models.
\begin{definition}[Plausibility action model]
A {\em plausibility action model} $\amodel = ( \Actions, \approx, \prec, \pre )$ consists of a {\em domain} $\Actions$ of {\em actions}, an {\em accessibility function} ${\approx}: \Agents \imp {\mathcal P}(\Actions \times \Actions)$, where each $\approx_\agent$ is an equivalence relation, a plausibility function ${\prec}: \Agents \imp {\mathcal P}(\Actions \times \Actions)$, and a {\em precondition function} $\pre: \Actions \imp \lang_X$, where $\lang_X$ is a logical language. The restriction of $\prec_\agent$ to an equivalence class of $\approx_\agent$ must be a prewellorder. A pointed plausibility action model is a {\em plausibility epistemic action}. 
\end{definition}

\begin{definition}[Update with a plausibility epistemic action]
The update of a plausibility epistemic state with a plausibility epistemic action is computed as the update without plausibilities (see Definition \ref{def.updateam}), except for the update of the plausibility function that is defined as:
\[ \begin{array}{lcl}  
(\state,\action) <_\agent (\stateb,\actionb) & \text{iff} & \action \prec_\agent \actionb \text{ or } \\ && \action =_\agent \actionb \text{ and } \state <_\agent \stateb
\end{array} \]
where $\equiv_\agent$ means equally plausible for agent $\agent$ (see after Definition \ref{def.pem}).
\end{definition}

\noindent See \cite{baltagetal.tlg3:2008} for details. Using these definitions, we propose the following plausibility epistemic actions for plausible (truthful and lying) public announcement that $\phi$ and for plausible (truthtelling, lying, bluffing) agent announcements that $\phi$ (by agent $\agent$ to agent $\agentb$). They are like the epistemic actions presented in Section \ref{sec.am}, but enriched with plausibilities. As before, the pointed versions of these action models define the appropriate epistemic actions.
\begin{definition}[Plausible (truthful and lying) public announcement] \ 
Plausible (truthful and lying) public announcements $!^\plau\phi$ and $\lie^\plau\phi$ are the epistemic actions defined by the, respectively, right/left point of the plausibility action model depicted as follows (where $\pre(!^\plau) = \phi$ and $\pre(!^\plau) = \neg\phi$).

\bigskip

\psset{border=2pt, nodesep=6pt, radius=2pt, tnpos=a}
\pspicture(-1,-1)(3.5,1)
$
\rput(0,0){\rnode{00}{\lie^\plau}}
\rput(3.05,0){\rnode{10r}{ \ }}
\rput(3,0){\rnode{10}{!^\plau}}
%\rput(0,-.5){\rnode{00c}{ \ }}
\rput(0,-1){\rnode{00a}{\neg\phi}}
\rput(3,-1){\rnode{10a}{\phi}}
\ncarc[arcangle=30]{->}{00}{10} \ncput*{\agent}
\ncarc[arcangle=-30,linestyle=dashed]{<->}{00}{10} \ncput*{\agent}
\nccircle[angle=270]{->}{10r}{.5} \ncput*{\agent}
\nccircle[linestyle=dashed,angle=270]{->}{10}{.7} \ncput*{\agent}
\nccircle[linestyle=dashed,angle=90]{->}{00}{.7} \ncput*{\agent}
$
\endpspicture
\
\end{definition}

\begin{definition}[Plausible agent announcement] 
Plausible (truthful, lying, and bluffing) agent announcements $!_\agent^\plau\phi$, $\lie_\agent^\plau\phi$ and $\bluff_\agent^\plau\phi$ are the epistemic actions defined by corresponding points of the plausibility action model consisting of points $!_\agent^\plau$, $\lie_\agent^\plau$, and $\bluff_\agent^\plau$, and such that $!_\agent^\plau \prec_\agentb \bluff_\agent^\plau \prec_\agentb \lie_\agent^\plau$, with universal access for agent $\agentb$ and identity access for agent $\agent$, and with preconditions as visualized below.

\bigskip

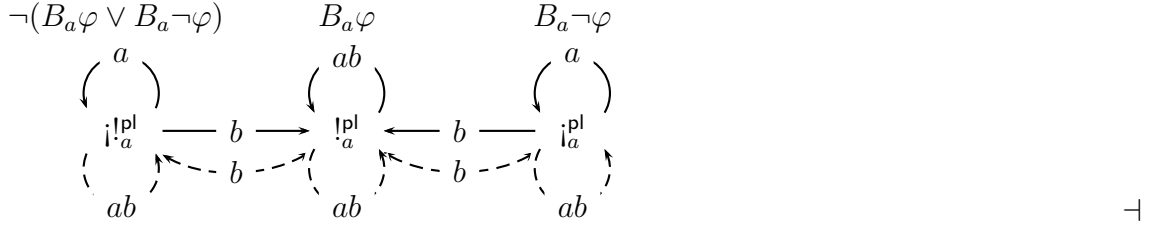
\psset{border=2pt, nodesep=5pt, radius=2pt, tnpos=a}
\pspicture(-3,-1.2)(6.5,2.1)
$
\rput(-1,-.5){\rnode{00c}{ \ ~ \ }}
\rput(2,-.5){\rnode{10c}{ \ ~ \ }}
\rput(5,-.5){\rnode{20c}{ \ ~ \ }}
\rput(-1,0){\rnode{00}{ \ \bluff_\agent^\plau \ }}
\rput(2,0){\rnode{10}{ \ !_\agent^\plau \ }}
\rput(5,0){\rnode{20}{ \ \lie_\agent^\plau \ }}
\rput(-1,1.5){\rnode{00b}{\neg (B_\agent \phi \vel B_\agent \neg \phi) \ }}
\rput(2,1.5){\rnode{10b}{B_\agent \phi}}
\rput(5,1.5){\rnode{20b}{B_\agent \neg \phi}}
\ncline{->}{00}{10} \ncput*{\agentb}
\ncline{<-}{10}{20} \ncput*{\agentb}
\ncarc[linestyle=dashed,arcangle=-30]{<->}{00}{10} \ncput*{\agentb}
\ncarc[linestyle=dashed,arcangle=-30]{<->}{10}{20} \ncput*{\agentb}
\nccircle{->}{00}{.5} \ncput*{\agent}
\nccircle{->}{10}{.5} \ncput*{\agent\agentb}
\nccircle{->}{20}{.5} \ncput*{\agent}
%\rput(-1,-2){\rnode{00bbl}{1}}
%\rput(2,-2){\rnode{10bbl}{0}}
%\rput(5,-2){\rnode{20bbl}{2}}
\nccircle[linestyle=dashed,angle=180]{->}{00}{.5} \ncput*{\agent\agentb}
\nccircle[linestyle=dashed,angle=180]{->}{10}{.5} \ncput*{\agent\agentb}
\nccircle[linestyle=dashed,angle=180]{->}{20}{.5} \ncput*{\agent\agentb}
$
\endpspicture
 \
\end{definition}
Agent $\agentb$'s equivalence relation is the universal relation. He cannot exclude any of the three types of announcement. Agent $\agentb$'s accessibility relation for belief expresses that he considers it most plausible that $\agent$ was telling the truth. It does not appear from the visualization that he considers it more plausible that $\agent$ was bluffing than that she was lying. The accessibility relation is the same as in the action model for agent announcement in Section \ref{sec.am}. Agent $\agent$'s accessibility relation is the identity on the action model. She knows whether she is lying, truthtelling, or bluffing.

As in the previous sections, an addressee rather assumes being told the truth than being told a lie or being bluffed to. But, unlike in the previous sections, we can now also encode more-than-binary preferences between actions. As lying seems worse than bluffing, we make it least plausible to interpret an announcement as lying, more plausible that it is bluffing, and most plausible that it is truthful. That is about as charitable as we can be as an addressee. This seems to be in accordance with pragmatic practice.

\begin{proposition}[Axiomatization and completeness]
The logics of plausible public lying and plausible agent lying have a complete axiomatization.
\end{proposition}
\begin{proof}
This follows from \cite[p.51]{baltagetal.tlg3:2008}. The so-called `Derived Law of Action-Conditional-Belief' is a reduction axiom for the belief postconditions of plausibility epistemic actions. This axiom involves modalities for belief, for knowledge, and conditional modalities for belief and for knowledge. The completeness proof consists of rewriting each formula in the logic to an equivalent formula in the logic of conditional belief (where conditional knowledge is definable as conditional belief).\footnote{Given a plausibility epistemic model $M$ and a state $\state$ in its domain, define $R_\agent^\phi$ as: $(\state, \stateb) \in R_\agent^\phi$ iff [$\state \sim_\agent \stateb$, $\stateb \leq_\agent \stateb'$ for all $\stateb'$ such that $\state \sim_\agent \stateb'$, and $M,\stateb\models \phi$], and define the semantics of $B_\agent^\phi$ as: $M,\state\models B_\agent^\phi \psi$ iff $M,\stateb\models \psi$ for all $\stateb$ such that $(\state,\stateb)\in R_\agent^\phi$. We now have that $[\lie^\plau \phi]B_\agent \phi \ \eq \ \neg \phi \imp B_\agent^\phi [! \phi]\psi$. The axioms for plausible agent lying are more complex (they also involve conditional knowledge modalities).} 
\end{proof}
As an example of a plausible agent announcement, we take Anne's lying announcement ``I know your number'' in the first lying scenario in the consecutive numbers riddle. It suffices to refer to the depicted execution for skeptical agents in the final paragraph of Section \ref{sec.unbel}. The accessibility relations for $B_\agent$ and $B_\agentb$ are as there (after the lie, it remains the same for the speaker $\agent$, and not for addressee $\agentb$). The equivalence relations for $\agent$ and $\agentb$ are the same before and after the lie. It is a coincidence that in this example a skeptical announcement is a plausible announcement. This is a consequence of the fact that the initial belief accessibility relations are equivalence relations.

\section{Conclusions and further research} \label{sec.further}

\paragraph*{Conclusions} We presented various logics for an integrated treatment of lying, bluffing, and truthtelling, where these are considered epistemic actions inducing transformations of epistemic models.  These logics abstract from the moral and intentional aspect of lying, and only consider fully rational agents and flawless and instantaneous information transmission. We presented versions of such logics that treat lies that contradict the beliefs of the addressee differently from those that don't, including a modelling involving plausible belief. Our main result are the various `agent announcement' logics wherein one agent is lying to another agent and wherein both are explicitly modelled in the system.

\paragraph*{Limitations}
There are limitations to our approach in view of the analysis and design of artificially intelligent agents. {\em Explicit agency} is missing. This lack is common in dynamic epistemic logics. We cannot distinguish agent $\agent$ truthfully announcing $\phi$ from an outsider observer (or a middleman, say, if it concerns a security protocol) announcing $B_\agent \phi$. Somewhat more involved, we cannot distinguish agent $\agent$ lying to agent $\agentb$ that $\phi$ from a non-public event differently observed by $\agent$ and $\agentb$, but that is not seen as enacted by $\agent$. Frameworks like ATL, ATEL, and STIT have more explicit agency. {\em Bounded rationality} is not modelled. An agent knows if it is lying. It cannot be mistaken in the sense of announcing an inconsistency of which it is not aware.
For not fully rational agents the computational cost for the liar to remember who he has been lying to, and about what, is real, and also the computational cost for the listener to check the reliability of new information. {\em Confidence of belief}, related to the previous, is not modelled. This requires a logic or mechanism of induction, wherein repeated observations or announcements, or by different agents, make the belief in its content stronger. In our modelling the agents do either believe the lie completely, or not at all (it abstracts from this process of increasing or decreasing confidence). {\em Distinguishing transmission noise from intentional noise} such as lies is not modelled, as transmission is assumed instantaneous.

\paragraph*{Further research}
We see the following research directions. 

{\em Common knowledge}: If agent $\agentb$ believes $\atom$ as the result of agent $\agent$ announcing $\atom$ to $\agentb$, we not only have $B_\agentb \atom$ but also $B_\agentb C_{\agent\agentb} \atom$: addressee $\agentb$ believes that he and the speaker $\agent$ now commonly believe that $\atom$. Common belief/knowledge operators allow for more refined preconditions. A good (and possibly strongest?) precondition for agent $\agent$ successfully lying that $\phi$ to agent $\agentb$ seems: \[ B_\agent \neg \phi \et \neg B_\agentb \neg \phi \et C_{\agent\agentb} ((B_\agent \phi \vel B_\agent \neg \phi) \et \neg (B_\agentb \phi \vel B_\agentb \neg \phi)) \]
In plain words, the speaker $\agent$ believes $\neg\phi$, the addressee $\agentb$ considers $\phi$ possible, and the announcer and addressee commonly believe/know that the announcer is knowledgeable about $\phi$ but the adddressee ignorant. Common knowledge logics with dynamic modalities for information change such as action models are straightforwardly axiomatized by way of conditional common knowledge \cite{jfaketal.lcc:2006}.

{\em Complexity}: The computational cost of lying seems a strong incentive against it. Aspects of this can also be modelled for perfectly rational agents. For example, in a different context, the computational cost of insincere voting in social choice theory \cite{conitzeretal:2009} is intractable in well-designed voting procedures, so that sincere voting is your best strategy. The complexity of model checking and satisfiability of the logics in this paper is unclear. It seems likely that the complexity of satisfiability of truthful and lying public announcement logic is in $\mathsf{PSPACE}$, applying similar results in \cite{lutz:2006} that builds on \cite{halpernetal:1992}. Also, it is unclear how to measure the complexity of a dynamic epistemic communication protocol that may involve lies.

{\em Histories}: Agents can detect lies, as in the consecutive numbers game, but we do not model that they adapt their strategies subsequently. History-based structures \cite{parikhetal:2003} allow the adapation of communication strategies to the number of detected (or suspected) past lies, in the dynamic epistemic setting of perfectly rational agents.

{\em Lying games}: Consider knowledge games wherein game states are epistemic states, the agents are players making informative moves that are epistemic actions, and wherein the players' goals are knowledge of facts or other formulas true in the resulting epistemic model. Such imperfect information games (Bayesian games) with modal logic are presented in \cite{agotnesetal:2011}. We wish to investigate {\em lying games} in this setting, such that each announcement can be either the truth or a lie, and where a high negative payoff is associated with the detection of a lie.

\paragraph*{Final frontiers}
In multi-agent systems with several agents one may investigate how robust certain communication procedures are in the presence of few liars, or in the presence of a few lies (as in Ulam games \cite{pelc:2002}). A challenge towards philosophy is how to model a liar's paradox. In a dynamic epistemic logic this is problematic. However, promising progress towards that is reported in \cite{liuetal:2012}.

% \section*{Acknowledgements}

% Yoram Moses

\bibliographystyle{plain}
\bibliography{biblio2012}

\end{document}